\documentclass[sigconf,authordraft=false]{acmart}

\usepackage{graphicx}
\usepackage{mathtools}
\usepackage{footnote}
\usepackage{float}
\usepackage{xspace}
\usepackage{multirow}
\usepackage{wrapfig}
\usepackage{framed}

\usepackage{footnote}

\makesavenoteenv{tabular}
\makesavenoteenv{table}

\newcommand{\LineComment}[1]{\hfill$\rhd\ $\text{#1}}

\newcommand{\calI}{{\mathcal{I}}}

\newcommand{\calD}{{\mathcal{D}}}

\newcommand{\one}{\boldsymbol{1}}

\DeclareMathOperator*{\argmax}{argmax}

\newcommand{\field}[1]{\mathbb{#1}}

\newcommand{\E}{\field{E}}

\newtheorem{lemma}{Lemma}
\newtheorem{theorem}{Theorem}

\newtheorem{definition}{Definition}

\newcommand{\order}{\ensuremath{\mathcal{O}}}

\newcommand{\xxnote}[3]{}
\ifx\hidenotes\undefined
  \usepackage{color}
  \renewcommand{\xxnote}[3]{\color{#2}{#1: #3}}
\fi

\AtBeginDocument{%
  \providecommand\BibTeX{{%
    \normalfont B\kern-0.5em{\scshape i\kern-0.25em b}\kern-0.8em\TeX}}}

\setcopyright{acmcopyright}

\usepackage{graphicx}
\usepackage{caption}
\usepackage{booktabs}
\usepackage{textcomp}
\usepackage{graphicx}
\usepackage{amsmath,mathrsfs,}
\PassOptionsToPackage{ruled, vlined, linesnumbered, noend}{algorithm2e}
\usepackage{algorithm2e}
\usepackage{footnote}
\usepackage{titlesec}

\usepackage{comment}

\titleformat*{\subsubsection}{\Large\it}
\copyrightyear{2020}
\acmYear{2020}
\setcopyright{acmcopyright}
\acmConference[HRI '20]{Proceedings of the 2020 ACM/IEEE International Conference on Human-Robot Interaction}{March 23--26, 2020}{Cambridge, United Kingdom}
\acmBooktitle{Proceedings of the 2020 ACM/IEEE International Conference on Human-Robot Interaction (HRI '20), March 23--26, 2020, Cambridge, United Kingdom}
\acmPrice{15.00}
\acmDOI{10.1145/3319502.3374806}
\acmISBN{978-1-4503-6746-2/20/03}

\begin{document}
\fancyhead{}

\title{Multi-Armed Bandits with Fairness Constraints for Distributing Resources to Human Teammates}


\author{Houston Claure}
\affiliation{%
  \institution{Cornell University}
  \city{Ithaca}
  \country{New York}}
\email{hbc35@cornell.edu}

\author{Yifang Chen}
\affiliation{%
  \institution{University of Southern California}
  \city{Los Angeles}
  \country{California}}
\email{yifang@usc.edu}

\author{Jignesh Modi}
\affiliation{%
  \institution{University of Southern California}
  \city{Los Angeles}
  \country{California}}
\email{jigneshm@usc.edu}

\author{Malte Jung}
\affiliation{%
  \institution{Cornell University}
  \city{Ithaca}
  \country{New York}}
\email{mfj28@cornell.edu}

\author{Stefanos Nikolaidis}
\affiliation{%
  \institution{University of Southern California}
  \city{Los Angeles}
  \country{California}}
\email{nikolaid@usc.edu }




\begin{abstract}

How should a robot that collaborates with multiple people decide upon the distribution of resources (e.g. social attention, or parts needed for an assembly)? People are uniquely attuned to how resources are distributed. A decision to distribute more resources to one team member than another might be perceived as unfair with potentially detrimental effects for trust. We introduce a multi-armed bandit algorithm with fairness constraints, where a robot distributes resources to human teammates of different skill levels. In this problem, the robot does not know the skill level of each human teammate, but learns it by observing their performance over time. We define fairness as a constraint on the  minimum rate that each human teammate is  selected throughout the task. We provide theoretical guarantees on performance and perform a large-scale user study, where we adjust the level of fairness in our algorithm. Results show that fairness in resource distribution has a significant effect on users' trust in the system.
\end{abstract}
\keywords{Reinforcement Learning; Fairness; Multi-Armed Bandits; Trust}

\maketitle

\section{Introduction}
For robots to function effectively in teams of multiple people, they have to be able to decide about the distribution of resources (e.g. assistance or social attention) \cite{jung2018robot,vazquez2016maintaining,short2017robot}.
Consider a factory robot that assists two workers by delivering parts needed for an engine assembly. One worker is experienced and fast, the other inexperienced and slow. How should a robot take expertise into account when dividing its assistance among workers to achieve optimal outcomes?



 
 The successful adoption of robots as parts teams requires not only the consideration of team performance (e.g. completion time, or cost), but also of team viability, that is the capability of team members to continue to work cooperatively over time \cite{barrick1998relating}. Team viability requires trust. A robot distributing more resources to one worker than another might be perceived as unfair and consequently unstrustworthy \cite{lee2018understanding}. In fact, previous work has shown that ignoring human preferences in task allocation can negatively affect users' willingness to work with the system~\citep{Gombolay2015CoordinationPreferences}. 
 Groom and Nass \citep{Groom2007CanInteraction} argue that our ability to build effective human-robot teams depends on a team's ability to build trust between all members of a team, and much work in human-robot interaction has focused on establishing perceived team fluency and trust in human-robot teams~\citep{gombolay2015decision, Shah2011ImprovedSystem,HuangAdaptiveHandovers,Baraglia2016InitiativeExecution,Chen2018PlanningCollaboration,Shu2018HumanTasks,desai2013impact,desai2012effects,kaniarasu2012potential}.

Here we focus on the notion of \emph{fairness} in resource distribution. We formalize how a robot can take individual expertise into account to maximize  team performance, while guaranteeing that each human teammate will be assigned a minimum rate of resources at any given time throughout the task. Our thesis is that, by accounting for fairness in resource allocation, we can significantly improve users' trust in the system. 

To this end, we cast the problem as a multi-armed bandit, where each human teammate is represented as an arm with an unknown reward function corresponding to their skill level. We then 
propose a \emph{multi-armed bandit algorithm with fairness constraints}, which builds upon the standard Upper Confidence Bound (UCB) algorithm~\citep{Auer2002Finite-timeProblem}. We propose a stochastic version of the algorithm, where a minimum pulling rate for each arm is satisfied in expectation, and a deterministic version where the constraint is strictly satisfied anytime throughout the task. 
We provide a new definition of regret and theoretical guarantees of performance in the form of regret bounds for both algorithms. 

To assess the effect of fairness on the users, we execute a large-scale user study on a Tetris game, where two players are sequentially assigned a batch of blocks by the algorithm. We selected the Tetris game, since it emulates collaborative tasks in human-robot interaction where a robot provides resources to human teammates~\cite{jung2018robot,Shah2011ImprovedSystem}, it provides a clear and transparent way to assess the participants' performance and it can model a wide range of task characteristics~\cite{lindstedt2013extreme}.

We implement the algorithm with three  levels of fairness, representing the required minimum allocation rate for each player: 25\%, 33\% and 50\%. Results show that fairness significantly affects the trust of the players that performed worse than their teammates: those in the 33\% condition trusted the system significantly more, compared to the 25\% condition. Surprisingly, we did not observe a decrease in performance in the fairer distributions, even though the stronger player was selected less frequently. On the contrary,  the median scores were higher when fairness increased. These results improve our understanding of the theory and implications of fairness in resource distribution in human-robot teams.

\section{Background}
\subsection{Distributive Fairness in Resource Distribution}

Fairness has been shown to be important for successful collaboration \cite{hackman1976motivation,crandall2018cooperating}. While fairness can be construed in many ways, we adopt a distributive perspective on fairness \cite{alexander1987role} and operationalize it consistent with \cite{lan2010axiomatic} as the degree of which resources are distributed equally to individuals within a group.  While an equal distribution of resources across all members within a group seems ideal, researchers  \cite{Lange1999TheOrientation,FismanWeGiving} have shown that inequalities are deemed appropriate, particularly when they optimize the outcome of the group. Adam's model on equity suggests that allocation decisions are deemed appropriate if they are in proportion with the input of the individual \cite{Adams1965InequityExchange,Stacy1963TowardInequity}. This model has been tested in various laboratory and real world scenarios suggesting that  allocation decisions in groups follow such a model \cite{Graf1971TheRetaliation,Walster1973NewResearch,Berscheid1968RetaliationEquity,Benton1971ProductivityChildren,Lane1971EquityRewards}.  On the other hand, perceived inequalities have a strong impact on individuals' behavior, often motivating them to act contrary to their rational self-interest with the goal of eliminating the inequality~\cite{Lecture2002PsychologicalIncentives,Camerer2003BehavioralInteraction}. Previous work has shown that perceived lack of fairness affects job satisfaction~\cite{Mcfarlin1992DistributiveOutcomes} and
can induce retaliation behavior from the affected party~\cite{skarlicki1997retaliation}. 

Interestingly, recent works have shown that individuals perceive fairness differently when decisions are made by an algorithm, compared to a human \cite{LeeAlgorithmicDivision,lee2018understanding}. As research in HRI advances, robots will be increasingly placed in complex environments where they will be making allocation decisions.  From allocating time, resources, and attention, these robotic systems will require an understanding of the impact their allocation decisions can have on individual and organizational dynamics. 

\subsection{Stochastic Multi-Armed Bandits}

The stochastic multi-armed bandits (MAB) framework without a minimum pulling rate requirement has been theoretically well studied. The gambler is tasked with choosing an  arm, $i$, from $K$ arms at each time step $t = 1,2,3,...,n$. At every time $t$, the gambler pulls an arm $i_t \in [K]$ while simultaneously the environment decides the reward vector $r_t \in [0,1]^K$ from a fixed distribution with expectation $ \E[r_t(i_t)] = \mu(i_t)$. The gambler, however, can only observe $r_t(i_t)$ but not the whole vector. Therefore, the gambler's goal is to pull the sequence of arms, based on the past information, that can maximize the overall accumulated reward. 

The best arm in hindsight is defined as $i^* = \argmax_{i\in[K]} \mu(i)$ and $\mu^* = \mu(i^*)$. We use regret to measure the performance of this algorithm, which is how worse our algorithm performs compared to the benchmark strategy -- always pulling the best arm in each step. \[Reg_T = T \mu^* - \sum_{t=1}^{T} \mu(i_t)\]

An optimal solution to such a problem was proposed as the Upper Confidence Bound (UCB). It was originally introduced by Lai and Robbins \cite{LaiAndherbertrobbins1985AsymptoticallyRules} and expanded by Agrawal \cite{Agrawal1995SampleProblem}. Building upon these works,  Auer, Cesa-Bianchi \& Fisher \cite{Auer2002Finite-timeProblem} introduced the Upper Confidence Bound Algorithm (UCB). In the most basic form of this algorithm, at each time $t$, we estimate the expected reward of each arm by using the mean of its empirical rewards in the past and the number of times it has been pulled, which gives us a confidence interval that the arm will lie in. Then the algorithm proceeds to pick the arm with the largest estimated expected reward.

This work has inspired a family of upper confidence bound variant algorithms for an array of different applications  \cite{Maillard2011ADivergences, KleinbergMulti-ArmedSpaces,Li2012ARecommendation,DudikEfficientBandits,Garivier2008OnMoulines}. For a review of these algorithms we point readers to \cite{Burtini2015ABandit}.

More recent work regarding multi-armed bandits has seen applications towards the improvement of human-robot interaction. Recent work has investigated using a MAB algorithm for the use of an assistive robotic system with the goal of exploring human preferences \cite{Chan2019TheBandit} and assisting human learning~\cite{PandyaHuman-AIBandits}.

Of particular relevance is very recent work on sleeping bandits with fairness constraints~\cite{li2019combinatorial}, in a setting where multiple arms can be played simultaneously and some arms may be unavailable. Fairness is defined as a minimum rate satisfied in expectation and at the end of the task, whereas in our work we  require the rate to be satisfied strictly and anytime throughout the task. Fairness in the context of MABs has also been studied in~\citet{joseph2016fairness}. The definition of fairness there is quite different, in that a worse arm should not be picked compared to a better arm, despite the uncertainty on payoffs. Their proposed algorithm  chooses two arms with equal probability, until it has enough data to deduce the best of the two arms. 

In addition, in parallel to our efforts, independent research \cite{patil2019achieving} has very recently proposed similar definitions of fairness, where a fairness-satisfaction phase -- that ensures that fairness is guaranteed -- is succeeded by a regret minimization phase. We refer the readers to this coming interesting work as well.

\section{Algorithm}
We propose two new algorithms with optimal regret bound guarantees. The original unconstrained UCB algorithm fails in ensuring ``fairness" because when time passes, a large set of ``bad" arms will hardly be used again. Both of the algorithms we propose are based on the unconstrained UCB algorithm, where we adopt the idea of estimating the expected reward of each arm by using the mean of its empirical rewards in the past and the number of times it has been pulled. We prove the following theorems in the Appendix.

\subsection{Strict-rate-constrained UCB Algorithm}
\begin{definition}
Let $S$ be any $K$-elements set whose elements are drawn from $[\frac{1}{v}]$ without replacement. Then define $g: S \rightarrow [K] $ as some one-to-one function.
\end{definition}

\begin{algorithm}[ht]

\caption{Strictly-rate-constrained UCB}
\label{algo:strict}
\begin{flushleft}
\textbf{Input:}  time horizon $T$, arm set $[K]$, minimum pull rate $v$\\

\textbf{Definition:} Denote $UCB_t(i) =  \frac{1}{t-1}\sum_{s=1}^{t-1} r_t(i) \one\{ i_s = i\} + 2\sqrt{\frac{\ln T}{n_{t-1}(i)}}$, and $\tau_j$ be the starting time of block $j$.

\SetAlgoVlined
\textbf{Initialize:} $t=1,j=1,\tau_1 = K+1$, $\tau_j = \tau_1 + \frac{j-1}{v}$. \\
\While{$t \leq K $}{
    Pull arm $i_t = t$\\
    $t\leftarrow t+1$. 
}
\For(\LineComment{$j$ indexes a block}){$j=1,2,3,\ldots$}{
    \While{$t < \tau_{j+1}$}{
        If $t - \tau_j + 1 \in S$, then pull the arm $i_t = g(t - \tau_j + 1)$,  \\
        Otherwise, pull the arm $ i_t = \argmax_{i \in [K]}UCB_t(i)$\\
        $t \leftarrow t+1$
    }
}
\end{flushleft}
\end{algorithm}

The algorithm divides $T$ into blocks with length $\frac{1}{v}$. The algorithm is flexible in that there are multiple choices of $S$ and $g$ that satisfy the minimum rate constraint. For example, if $v = \frac{1}{4}$ and $K = 2$, we can choose $S = \{1,3\}$ and $g(1) = 1, g(3) = 2$, which means we always pull arm $1$ at $\tau_j$ and arm $2$ at $\tau_j+2$ for all $j$ (see implementation in Section~\ref{sec:user_study}).

This algorithm guarantees that \textsc{in practice} the  pulling rate at any time for each arm is at least $v-\epsilon$, by fixing certain time slots where the algorithm will pull the prescheduled arms. Here $\epsilon = 1/t$.\footnote{ We can prove this by observing that at time $t$, the arm $i$ will be pulled at least $\lfloor tv \rfloor$ times according to the pre-schedule. So the pulling rate will be $\frac{\lfloor tv \rfloor}{t} \geq \frac{tv - 1}{t} = v-\frac{1}{t}$.} In other time slots, the algorithm will behave just like the standard UCB.


Now the benchmark strategy for pulling an arm is always pulling the best arm in those non-prescheduled time slots. So the regret definition in this case becomes:
\begin{align*}
    Reg_T = \E_{env} \left[\sum_{t\in \calI} r_t(i^*) - r_t(i_t) \right]
\end{align*}
where $\calI$ is all the non-prescheduled time slots among $T$.

\begin{theorem}
By running Alg.~\ref{algo:strict}, we obtain the regret bound that is close to the original unconstrained UCB, 
\begin{align*}
    Reg_T \leq \sum_{i:\Delta_i > 0}  \left[ \frac{16\ln T}{\Delta_i}\left( \frac{1-Kv}{1-(K-1)v} \right)+2(1-Kv)^2\Delta_i  \right]+ \order (K)
\end{align*}
If $\Delta_i \in [0,1]~\forall i$, we also get the worst case guarantee, 
\begin{align*}
    Reg_T \leq \order (\sqrt{TK\ln T} + K\ln T ) \\
\end{align*}
\end{theorem}

\subsection{Stochastic-rate-constrained UCB Algorithm }

\begin{algorithm}[ht]
\caption{Stochastic-rate-constrained UCB}
\label{algo:stochastic}
\begin{flushleft}
\textbf{Input:}  time horizon $T$, arm set $[K]$, minimum pull rate $v$\\
\textbf{Definition:} Denote $UCB_t(i) =  \frac{1}{t-1}\sum_{s=1}^{t-1} r_t(i) \one\{ i_s = i\} + 2\sqrt{\frac{\ln T}{n_{t-1}(i)}}$.\\

\SetAlgoVlined
\textbf{Initialize:} $t=1,j=1,\tau_1 = K+1$, $\tau_j = \tau_1 + \frac{j-1}{v}$. \\
\While{$t \leq K $}{
    Pull arm $i_t = t$\\ 
    $t\leftarrow t+1$. 
}
\For{$t=K+1,K+2,K+3,\ldots$}{
        With probability $1-Kv$, pull the arm $ i_t = \argmax_{i \in [K]}UCB_t(i)$,  \\
        Otherwise, uniformly pull an arm $i_t$ from all $K$ arms \\
}
\end{flushleft}
\end{algorithm}
This algorithm guarantees that the \textsc{expected} pulling rate at any time for each arm is at least $v$. Instead of rescheduling some arms as in the deterministic algorithm above, this algorithm introduces some randomness. At each time $t$, we ensure that each arm has at least $v$ probability to be pulled; while with $1-Kv$ probability, the algorithm will again pull the arm with the best UCB bound. We denote this distribution over arms as $p_t$ where $p_t(\argmax_{i \in [K]}UCB_t(i))=(1-Kv)+v$ and $p_t(i) = v, \forall i \in [K]\setminus{\argmax_{i \in [K]}UCB_t(i)}$. 

In this case, the benchmark strategy is pulling the best estimated arm with probability $(1-Kv)$ at time $t$, otherwise uniformly drawing a random arm. We present this strategy with the distribution $p^*$ over $K$ arms where $p^*(i^*)=(1-Kv)+v$ and $p^*(i) = v, \forall i \in [K]\setminus{i^*}$. So the regret definition in this case becomes:
\begin{align*}
    Reg_T = \E_{env,learner} \left[ \sum_{t=1}^T \E_{i_t\sim p^*}[r_t(i_t)] - \E_{a_t \sim p_t}[r_t(i_t)] \right] 
\end{align*}

\begin{theorem}
By running Alg.~\ref{algo:stochastic}, we obtain the regret bound that is close to the original unconstrained UCB, 
\begin{align*}
    Reg_T 
    \leq \sum_{a:\Delta_i > 0}  \left[ \min \left\{\frac{16\ln T}{\Delta_i}+(1-Kv)\Delta_i, (1-Kv)\Delta_iT \right\} \right] 
\end{align*}
If $\Delta_i \in [0,1]~\forall i$, we also get the worst case guarantee,
\begin{align*}
   Reg_T < \order (\sqrt{TK\ln T} + K\ln(T))
\end{align*}
\end{theorem}

\section{Evaluation} \label{sec:user_study}

\begin{figure*}[ht]
    \centering
    \includegraphics[width=1.0\linewidth]{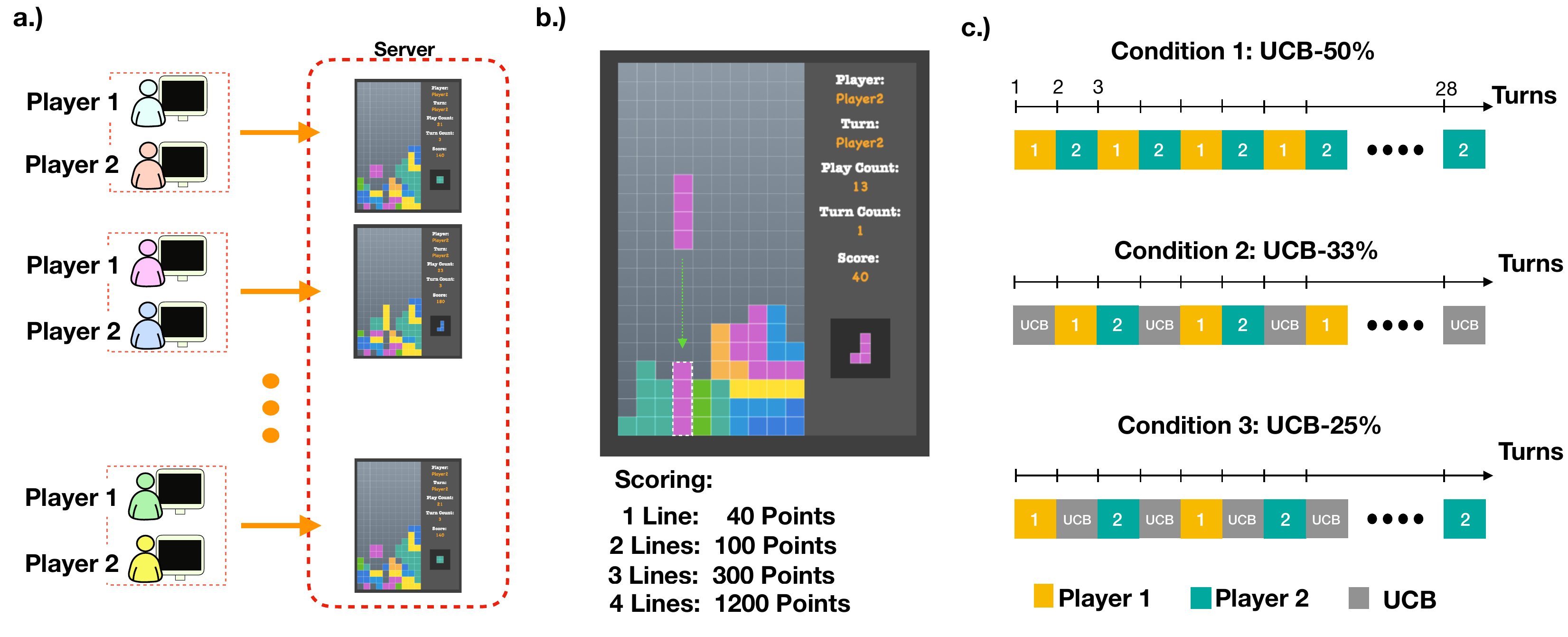}
    \hfil
    \caption{(a) Pairs of two remote human participants were connected to our cooperative Tetris game online. (b) The Tetris game followed the standard rules with two slight modifications. A different scoring metric as shown and only one participant had access to control the pieces per turn. (c) A visual representation of the three separate patterns that each condition offered.}
    \label{setup}
\end{figure*}
To evaluate our algorithm we conducted an online user study. We asked teams of two participants each to collaborate with a robot in completing a cooperative game (a modified version of Tetris). The robot's task is to decide which team-mate receives the next block to place. Similar to other human-robot collaboration tasks (e.g. \cite{hayes2015effective, jung2018robot, gombolay2015decision,shah2011improved,Gombolay2015CoordinationPreferences}), the robot's role is that of a task or resource allocator as it provides a resource needed to complete a task (in this case Tetris blocks) to participants.

We designed a between-teams study with three conditions of increasing distributional fairness constraints maintained by the UCB Algorithm (Alg.~\ref{algo:strict}): (UCB-25\%,UCB-33\%, UCB-50\%).\footnote{ We did not include a UCB-0\% condition, since in our pilot studies in the team-Tetris task the UCB-0\% and UCB-25\% had similar allocation of pieces, because of the variance in the scores and the exploration inherent in the unconstrained UCB-0\%.} To reduce variance from sampling, we implemented the  strict-rate-constrained UCB Algorithm.\footnote{ \url{https://github.com/icaros-usc/MAB_Fairness}}

We characterize the player that has the best performance of the two, as observed at the end of the task, as \emph{strong} and the the other player as \emph{weak.} The challenge of balancing between choosing the historically best player or a sub-optimal player allows us to investigate the impact of the system's decision on team performance, perceived fairness and trust in the system. 
 
Based on prior work which has shown that people react strongly to fairness in resource allocation \cite{brosnan2014evolution, lee2018understanding}, with especially strong reactions occurring for the disadvantaged party (e.g. ~\cite{skarlicki1997retaliation}), we expect that distribution rates (fairness) will have a significant effect on perceived fairness and trust in the system of the weak players. Specifically, as the distribution rates favor the stronger individuals at an increasing rate (UCB-50\%,UCB-33\%,UCB-25\%, respectively) we expect fairness perceptions and trust in the system to decrease \textbf{(H1)}. Furthermore, prior literature has shown that fairness in resource distribution has implications for team performance \cite{Colquitt2001JusticeResearch}. In our case, the fairer distributions favor the weak players, since they impose a constraint on the minimum number of pulls for both players. We expect that this will result in worse performance, compared to the less fair distributions that favor the strong player of the team \textbf{(H2)}.


\subsection{Methodology}
\subsubsection{Participants\nopunct}\hfill\

We recruited 290 participants from Amazon Mechanical Turk (AMT) and paid \$1.00 for their participation in the task. We selected participants with previous ratings of 95\% or higher. 8 data points were removed, since their AMT unique ID did not match the one given on Qualtrics. The final dataset contained N = 94, 98, 90 participants for UCB-50\%, UCB-33\%, and  UCB-25\% respectively (156 female, 124 male, 1 other, 1 did not disclose). The average age of participants was 36 years old (SD = 11). Participants were recruited if they could speak English, were from the United States, and had previous ratings of 95\% or higher. Of the 282 participants, 6 of them reported to have never played Tetris before. 


\subsubsection{Task: Collaborative Tetris\nopunct}\hfill\

 Building on a task developed by \citep{jung2018robot} we developed a collaborative Tetris game that paired teams of two people to complete a game of Tetris together with our MAB algorithm. The goal for each team was to achieve the highest score possible. At each defined time step, the algorithm decides which team-mate should have control over the falling pieces--thus only one human player has control over the set of Tetris blocks at each time step, while both players observe the moves of the blocks at all times. 
 
 
We chose Tetris as a collaborative task, since it has been shown to effectively model a broad range of common task characteristics  \cite{lindstedt2013extreme}, having been used as a testbed for several other studies (e.g. \cite{kirsh1994distinguishing, haier1992regional}). It emulates previous settings where a robot assists users in a collaborative task \cite{jung2018robot,shah2011improved}, providing a transparent and unambiguous way for the participants to observe each other's performance.
  

  \begin{figure*}
    \centering
    \includegraphics[width=1.0\linewidth]{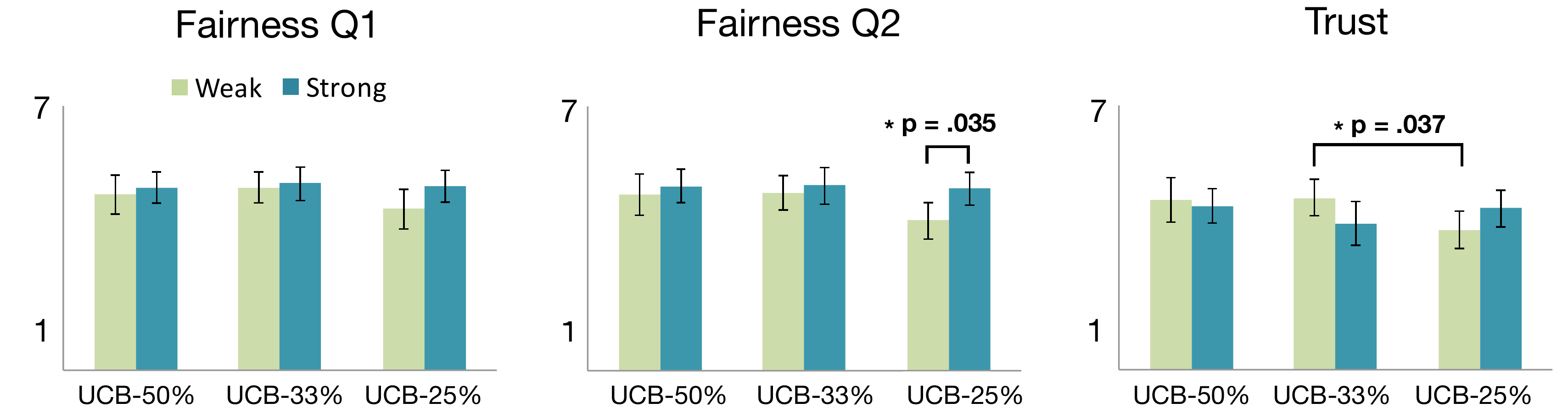}
    \hfil
    \caption{ Responses to the subjective questions, grouped by player performance across each condition. Error bars represent the 95\% confidence intervals. 
    }
    \label{split sub}
\end{figure*}

We formally define our scenario as follows. The number of players in each game is set as $P = \{1,2\}$ over a time horizon of $T = 30$. At each time step $ t \leq T $, seven consecutive Tetris pieces are allotted to a player $p \in P$. In the turns where the UCB algorithm was run, we used as reward $r_{p,t}$: 
\begin{align*}
  r_{p,t} = \frac{S_{p,t}}{M*n_{p,t}}  
\end{align*}
 where $S_{p,t}$ is the score achieved by player  $p$ up to turn $t$, $n_{p,t}$ is the number of turns of that player and $M$ is a maximum value that we selected for normalization. After multiple pilot sessions we empirically set $M$ to 300. 
 
 We defined each time step as a set of seven consecutive falling pieces that only the selected player could control. Our pilot sessions showed that allowing control of seven consecutive pieces together with limiting the width of the Tetris board prevented behaviors where one would place the blocks in such as way that prepares the groundwork for their partner to get the rewards. Observing the players' behavior in the pilot sessions, as well as their responses to questionnaires at the end of the study, confirmed that recorded scores matched observed performance.

\subsubsection{Procedures\nopunct} \hfill\

 Upon providing informed consent and entering basic demographic information,  AMT participants were instructed that they would be paired with a human partner and a robot that would decide who has control of the falling pieces and that the objective was to obtain the largest possible team score. Following standard Tetris rules, a player could rotate, speed up, or drop each falling piece. At the end of the time step the 50\%, 33\%, or 25\% UCB, depending on condition, algorithm would run to select the next player.  

Figure~\ref{setup}(c) shows the pattern of the distribution that was seen across each condition. This pattern was repeated for 30 time steps, with the exception of the first two time steps where each player played once. Each team was exposed to 210 pieces total. A code was given to participants upon the completion of the 30 rounds which enabled them to continue the Qualtrics survey. 




\begin{figure*}[t]
    \centering
    \includegraphics[width=1.0\linewidth]{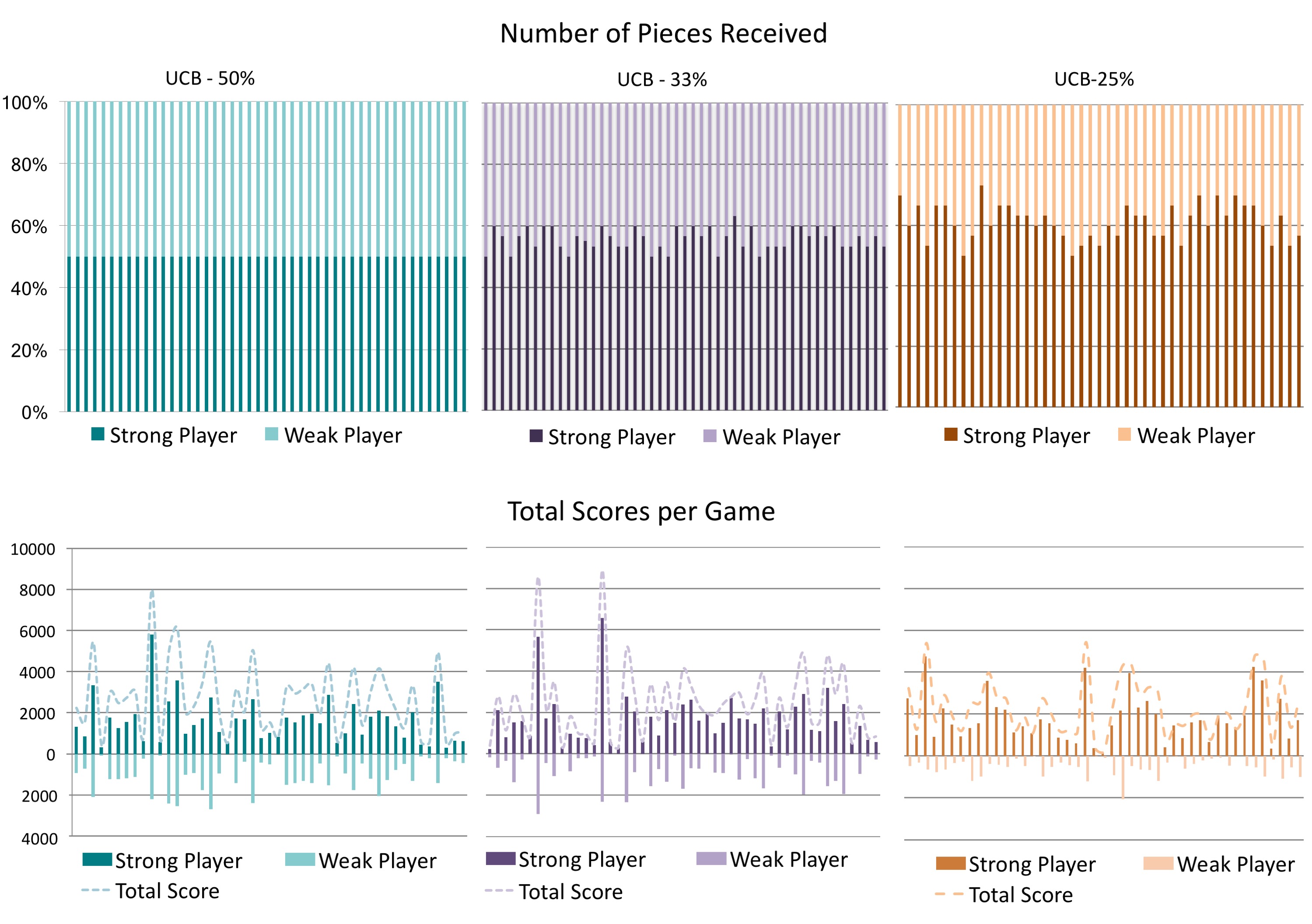}
    \hfil
    \caption{
      (Top) Percentage of the number of pieces that each player received. Each bar represents a separate game. (Bottom) Plots of the total scores that each player individually and both players together achieved. Each bar represents a separate game.
    }
    \label{pieces_and_scores}
\end{figure*}

 \begin{table}[ht]
    \centering
    \caption{Subjective Metrics }
      \vspace{-1em}
    \includegraphics[width=1.0\linewidth]{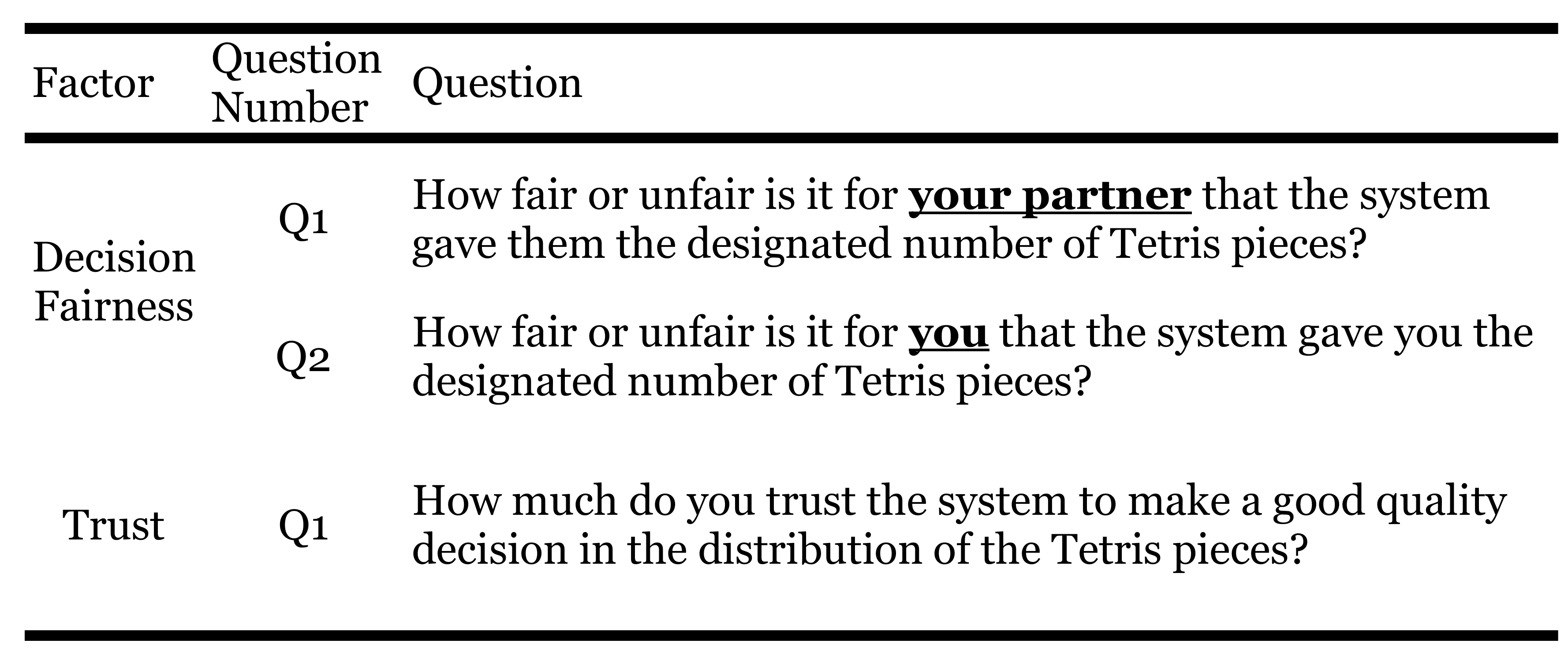} 

       \label{tab:questions}

\end{table}

\subsubsection{Measures\nopunct} \hfill\

\noindent\textbf{Subjective: }To measure levels of perceived \textit{decision fairness } and \textit{trust} we adapted survey questions from \cite{lee2018understanding} (Table~\ref{tab:questions}).  Each response was measured on a seven-point Likert scale. Finally, we asked an open ended question, ``In your own words, describe the strategy that you think the robot used to distribute the Tetris pieces.'' 

\noindent\textbf{Objective: }Information regarding an individual's performance was stored in a database during game play. We collected each player's individual score  as well as the number of turns that was allocated to them. Additionally, we obtained the total score that each team accumulated at the end of the game play. Figure~\ref{setup}(b) shows the scoring convention that we used.


\subsection{Results}

\begin{figure*}
    \centering
    \includegraphics[width=1.0\linewidth]{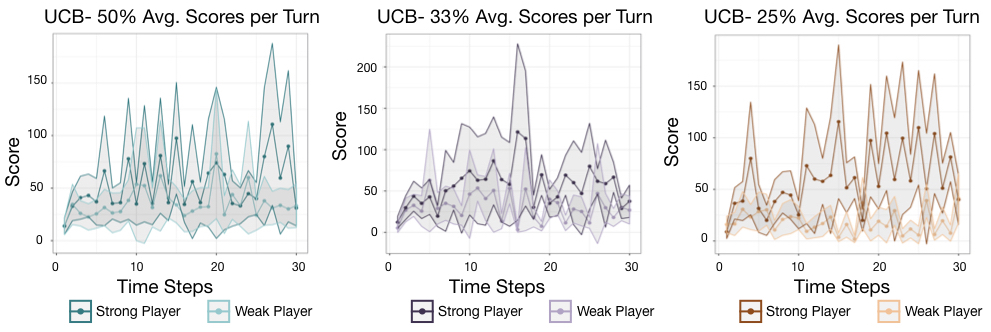}
    \hfil
    \caption{ Average Score per turn across the different conditions with 95\% confidence intervals. Each point represents the average score of all participants at that time step split into weak and strong participants. }
  
    \label{scores_turn}
\end{figure*}

\begin{figure}
    \centering
    \includegraphics[width=1.0\linewidth]{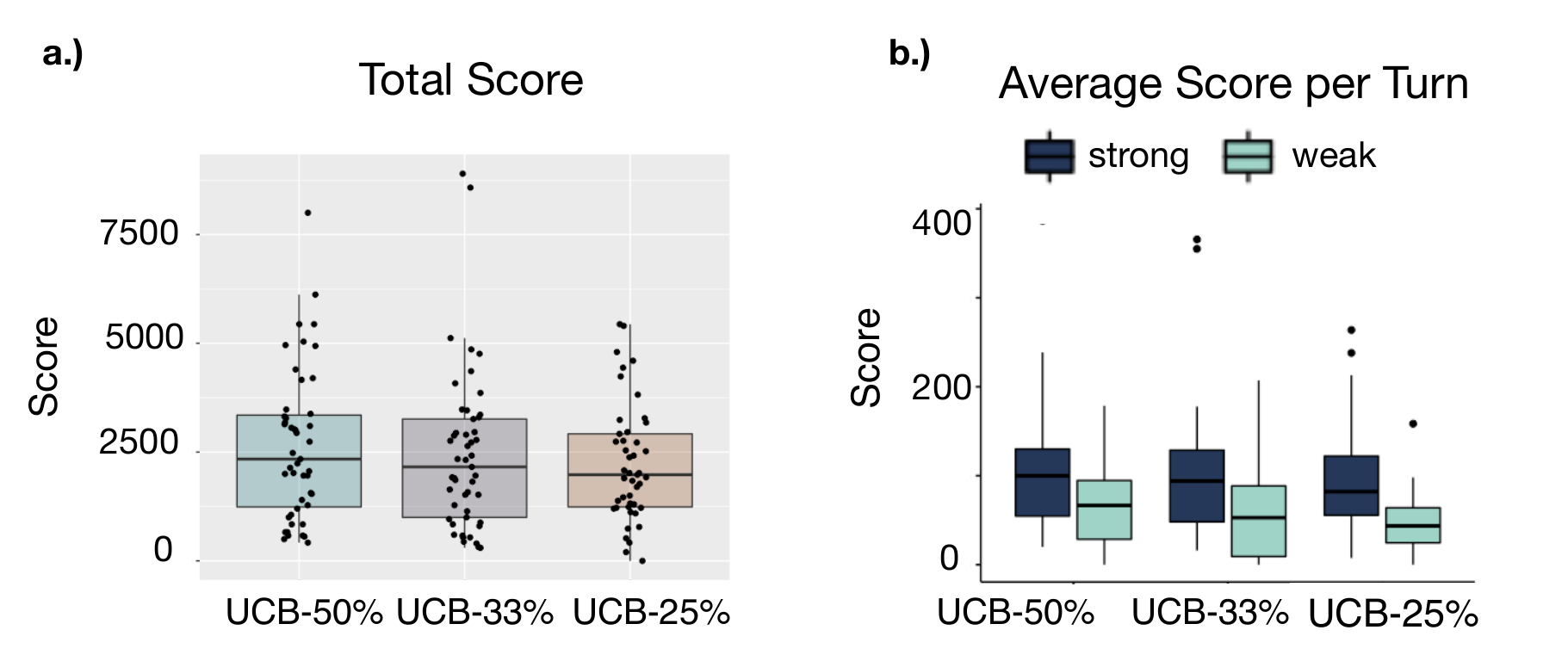}
    \hfil
    \caption{ (a) Total scores for each condition. (b) Average score per turn for each condition. 
    }
   
    \label{objective}
\end{figure}

\noindent\textbf{Subjective}:
  We grouped subjective responses of each pair of players based on their comparative performance in the game (Figure~\ref{split sub}). We focus the analysis on the weak players, that is the players that performed worse than their teammate. We present the responses of the strong players as well for completeness.
  

  A one-way ANOVA was performed for weak players across all conditions (UCB-50\% vs. UCB-33\% vs. UCB-25\%)  for each subjective metric. Analysis indicates a significant effect of the reported trust score of the weak players across the three conditions $(F(2,138)=3.172, p = 0.025)$. A Tukey HSD with adjusted p-values demonstrated higher trust $(p= 0.037)$ towards the system running the UCB-33\% compared to the UCB-25\%. While trust scores in UCB-50\% were higher than in the UCB-25\%, the difference was not significant $(p = 0.061)$. Differences between all other factors were not significant. 

 
\emph{Post-hoc Analysis.} We observed a noticeable difference in the responses between the strong and the weak players for different fairness conditions (Fig.~\ref{split sub}). Therefore, we conducted a post-hoc analysis to assess whether there were significant differences in the responses between the weak and the strong players \textit{within} each condition. Indeed, a 2 x 3 ANOVA with strength (weak vs. strong) and rate (UCB-50\% vs. UCB-33\% vs. UCB-25\%) showed a main effect of players' strength for Decision Fairness Q2 $(F(1,276) = 4.778, p = 0.0297)$. There were no interaction effects. Post-hoc comparison with Bonferroni corrections looking at strength within the different fairness levels, showed that weak players (M = 3.97, $\sigma $= 1.68) reported significantly lower ratings on fairness (Q2) than their strong counterparts (M = 4.82, $\sigma$ = 1.49) in the UCB 25\% condition ($p = 0.035$), which was the least fair condition. We observed no significant difference in perceived fairness between strong and weak players in the other two conditions.

In summary, there was a significant difference  between the weak and the strong players in their perception of fairness in the least fair condition (UCB-25\%), and reducing the minimum rate from 33\% to 25\% had a negative effect on the trust of weak players. On the other hand, Figure~\ref{split sub} shows that trust scores between the UCB-50\% and UCB-33\% conditions were comparable. 

To interpret these results, we observe the number of pieces received (arm pulls) for each condition in Figure~\ref{pieces_and_scores}(top). In the UCB-50\% condition, all players received the same number of pieces regardless of their performance. In the UCB-33\% condition, while the strong players received more pieces, the difference with the weak players was small. On the other hand, in the UCB-25\% condition there were several games where the weak players received less than 30\% of the pieces, resulting in lower reported trust in that condition.

We also examined participants' perceptions of resource distributions. Several participants in the UCB-33\% and UCB-25\% conditions noted how the system appeared to favor the "stronger" player during the gameplay:

\emph{"I felt the more competent player was given more turns. Which makes sense but was why it felt unfair."}

Participants in the UCB-50\% condition noted how the system gave each team member an equal number of turns:

\emph{"I think it was even, it made us take turns one after the other, enough that it made me feel I was making an equal contribution to the game."}


\noindent\textbf{Objective}: A one-way ANOVA on the performance of the two-player teams across the three conditions indicated no statistical significance. In fact, Figure~\ref{objective}(a) shows that the medians of the total scores were higher for increasing levels of fairness. Plotting the individual scores of the players for each game in Figure~\ref{pieces_and_scores}(bottom)  illustrates this tendency as well.

This result does not match our initial hypothesis. To interpret this result, we plot the average scores per turn for each condition in Figure~\ref{objective}(b). The average scores indicate how well the players performed on average every time they took a turn. Interestingly, we see that the distribution of the weak players' scores shifts towards lower scores as fairness decreases. While this result warrants further investigation, it indicates that assigning significantly less pieces to one of the players may negatively affect their performance, in addition to reducing their trust in the system. It showcases the importance of fairness when making resource distribution decisions. 



We further tested for learning effects, since players may get better at the game over time. We fit a linear mixed effects model for the fixed effects of strength (weak vs. strong) and time as a continuous variable while including participants as random effects. We found a statistically significant increase in player scores over time in the UCB-50\% (F(1,1313)= 6.69, p= 0.010) and UCB-25\% condition (F(1,1285)= 5.50, p= 0.019). We found no statistically significant increase for participants in the UCB-33\% condition (F(1,1359)= 0.140, p= 0.708). Indeed, Figure~\ref{scores_turn} shows that the average player scores tended to increase after the first few turns. These changes in performance occur in the first half of the game suggesting that learning effects did not drive our conclusions. Our goal was not to find the "true" stronger or weaker player within each game of Tetris, rather it was to asses how varying levels of distributions affect the team's performance within a given time period. The notions of "strong" and "weak" are with respect to their distribution of scores within the given 30 turns. While we have assumed fixed (albeit stochastic) reward distributions,  Sliding-Window-UCB-based algorithms~\cite{garivier2011upper,wei2018abruptly,cheung2018learning} have been proposed for evolving distributions.We suggest as future work extending these algorithms to account for fairness; providing  theoretical guarantees for these algorithms would follow the same reasoning as in the proofs that we include in the supplemental material.

\section{Discussion}
We proposed a novel algorithm for a robot's resource distribution for human robot collaboration scenarios that include multiple human team members. Specifically, our MAB variant algorithm aimed to maintain a level of fairness by administering a minimum rate constraint limiting the number of times an individual may be assigned a resource.

An evaluation of the algorithm in a collaborative Tetris game showed that optimizing not only based on team member expertise, but also based on distributive fairness, lead to higher trust without a decline in team performance. Specifically, our study revealed a statistically significant difference between the weak players in the UCB-33\% and the UCB-25\% conditions, partially supporting H1. We conjecture that this is due to the fact that weaker players in the UCB-25\%  were exposed to longer waiting periods and smaller number of turns than that of weaker players in the UCB-33\%. This highlights the result of unfair distribution on trust, particularly that of weaker performing individuals.  


Contrary to our expectations, our study did not find any differences in perceived fairness across the UCB-50\% , UCB-33\%, and UCB-25\% conditions. Since the goal was stated as to maximize the team's overall score, some participants may have seen it as appropriate for the stronger player to receive more turns. For these participants, distributions that favoured one of the players might have been seen as procedurally fair. Procedural fairness "refers to the perception by the individual that a particular activity in which they are a participant is conducted fairly" \citep{culnan1999information}. In other words the algorithmic or procedural nature of the distribution may have contributed to a perception of fairness irrespective of the actual distribution. This interpretation is consistent with interview reports from some participants, for instance: 

\emph{"I think it was judging that the other player was much better than me and thus they deserved to play more pieces than myself. "}

\noindent

Additionally, we found no statistically significant differences in perceived fairness between the strong and weak participants within the UCB-50\%, UCB-33\%  conditions. In these conditions the number of turns between the strong and weak participants did not differ much (Figure \ref{pieces_and_scores}), suggesting that many participants did not perceive the distribution as unfair as they experienced almost equal participation in the game (and exactly equal participation in the UCB-50\% condition). On the other hand, the UCB-25\% condition saw a larger difference in the number of turns between weak and strong players and that lead to a statistically significant difference between the two groups when asked about the perceived fairness of the distribution (Decision Fairness Q2).
 
Finally, we see that the performance of the teams was not significantly impacted by the difference in resource distribution, which does not support H2. A possible reason is that, when weaker participants were limited in the number of turns and participation within the team, they might have put lower effort, affecting the overall performance, as shown by the low average scores of the weak players in UCB-25\% in Figure~\ref{objective}(b).

These results are in line with previous work that highlights the importance of perceived allocation fairness on trust \cite{Cohen-Charash2001TheMeta-analysis,Colquitt2001JusticeResearch}. As robots are increasingly placed in contexts where they are faced with allocation decisions, our work contributes important initial insights on a robot's impact on groups.



\subsection{Limitations}

In interpreting the findings from our evaluation, we need to address several limitations. First, the evaluation of our algorithm was conducted in the context of an online game which opens question about the generalizability of our findings to human-robot teamwork with physically embodied robots. While we cannot say how physical robot characteristics influence perceptions of trust and fairness, our overall approach to focus on resource distribution matches a scenario that has been proposed by Jung and colleagues \citep{jung2018robot} in the context of physical human robot collaboration. Moreover, their work highlights how a robot's distribution of resources impacts team satisfaction yet does not influence the team's performance on a task, which is consistent with the results of our collaborative Tetris game. Our work extends this previous work \citep{jung2018robot}, by examining reactions to an autonomous system rather than one based on Wizard of Oz control.

In our theoretical contribution, we have assumed the same minimum allocation rate for all arms. One can easily extend the proposed algorithms and theoretical results by defining a vector of different rates for each arm. By  following the same reasoning as in the proofs provided in the supplemental material, the reader can verify that compared to an oracle that satisfies the same constraints, the regret will not be worse than unconstrained UCB.

A limitation of our paper is that Algorithm~\ref{algo:strict} allows for a set of possible schedules that satisfy the minimum rate constraint, based on our choice of $S$ and $g$. For instance, in the UCB-25\% condition we chose to play the arm with the highest UCB bound in the second and fourth timeslot, but we could also select the first and second timeslot. In fact, given a minimum rate $v$ there are $\frac{(1/v)!}{(1/v-K)!}$ permutations, and we have not captured the effect of different schedules within a fairness condition. 

\section{Conclusion}
This work explored the impact of resource allocation fairness by a system on human teams. We formulated the problem of distributing resources within a team as a multi armed bandit problem and developed two algorithms that constrains the number of resources an individual may acquire. Applying our algorithm with three distinct constraint rates (25\%, 33\%, 50\%) as independent variables, we explored team member perception's of trust and fairness through a cooperative Tetris game. Results from our user study suggest that fairness in resource allocation can influence trust of weaker performing individuals but may not have an impact on the overall performance of the team. 

Our work adds a new and novel framework for studying an important yet underexplored topic: robots in human teams. We demonstrated how fairness, as experienced from individuals within a team, did not vary across the different allocation rates. Results  indicate a difference in perception of fairness between weak and strong performing players in the case where the allocation more heavily favored the stronger player, while we did not observe a significant difference in performance. 

As robots are becoming more and more commonplace in everyday contexts at work and at home, they increasingly face situations that involve interactions with groups or teams of people. When interacting with multiple people robots have to make decisions about resource distributions. Our work demonstrates that such decisions can be made in ways that not only take task concerns into account but also human concerns of fairness.

\bibliographystyle{ACM-Reference-Format}
\bibliography{references_nw.bib}


\begin{thebibliography}{60}


\ifx \showCODEN    \undefined \def \showCODEN     #1{\unskip}     \fi
\ifx \showDOI      \undefined \def \showDOI       #1{#1}\fi
\ifx \showISBNx    \undefined \def \showISBNx     #1{\unskip}     \fi
\ifx \showISBNxiii \undefined \def \showISBNxiii  #1{\unskip}     \fi
\ifx \showISSN     \undefined \def \showISSN      #1{\unskip}     \fi
\ifx \showLCCN     \undefined \def \showLCCN      #1{\unskip}     \fi
\ifx \shownote     \undefined \def \shownote      #1{#1}          \fi
\ifx \showarticletitle \undefined \def \showarticletitle #1{#1}   \fi
\ifx \showURL      \undefined \def \showURL       {\relax}        \fi
\providecommand\bibfield[2]{#2}
\providecommand\bibinfo[2]{#2}
\providecommand\natexlab[1]{#1}
\providecommand\showeprint[2][]{arXiv:#2}

\bibitem[\protect\citeauthoryear{Adams}{Adams}{1965}]%
        {Adams1965InequityExchange}
\bibfield{author}{\bibinfo{person}{J.~Stacy Adams}.}
  \bibinfo{year}{1965}\natexlab{}.
\newblock \showarticletitle{{Inequity In Social Exchange}}.
\newblock \bibinfo{journal}{\emph{Advances in Experimental Social Psychology}}
  (\bibinfo{year}{1965}).
\newblock
\showISSN{00652601}
\urldef\tempurl%
\url{https://doi.org/10.1016/S0065-2601(08)60108-2}
\showDOI{\tempurl}


\bibitem[\protect\citeauthoryear{Agrawal}{Agrawal}{1995}]%
        {Agrawal1995SampleProblem}
\bibfield{author}{\bibinfo{person}{Rajeev Agrawal}.}
  \bibinfo{year}{1995}\natexlab{}.
\newblock \showarticletitle{{Sample mean based index policies by O(log n)
  regret for the multi-armed bandit problem}}.
\newblock \bibinfo{journal}{\emph{Advances in Applied Probability}}
  \bibinfo{volume}{27}, \bibinfo{number}{4} (\bibinfo{date}{12}
  \bibinfo{year}{1995}), \bibinfo{pages}{1054--1078}.
\newblock
\showISSN{0001-8678}
\urldef\tempurl%
\url{https://doi.org/10.2307/1427934}
\showDOI{\tempurl}


\bibitem[\protect\citeauthoryear{Alexander and Ruderman}{Alexander and
  Ruderman}{1987}]%
        {alexander1987role}
\bibfield{author}{\bibinfo{person}{Sheldon Alexander} {and}
  \bibinfo{person}{Marian Ruderman}.} \bibinfo{year}{1987}\natexlab{}.
\newblock \showarticletitle{The role of procedural and distributive justice in
  organizational behavior}.
\newblock \bibinfo{journal}{\emph{Social justice research}}
  \bibinfo{volume}{1}, \bibinfo{number}{2} (\bibinfo{year}{1987}),
  \bibinfo{pages}{177--198}.
\newblock


\bibitem[\protect\citeauthoryear{Auer and Fischer}{Auer and Fischer}{2002}]%
        {Auer2002Finite-timeProblem}
\bibfield{author}{\bibinfo{person}{Peter Auer} {and} \bibinfo{person}{Paul
  Fischer}.} \bibinfo{year}{2002}\natexlab{}.
\newblock \showarticletitle{{Finite-time Analysis of the Multiarmed Bandit
  Problem*}}.
\newblock   \bibinfo{volume}{47} (\bibinfo{year}{2002}),
  \bibinfo{pages}{235--256}.
\newblock


\bibitem[\protect\citeauthoryear{Baraglia, Cakmak, Nagai, Rao, and
  Asada}{Baraglia et~al\mbox{.}}{2016}]%
        {Baraglia2016InitiativeExecution}
\bibfield{author}{\bibinfo{person}{Jimmy Baraglia}, \bibinfo{person}{Maya
  Cakmak}, \bibinfo{person}{Yukie Nagai}, \bibinfo{person}{Rajesh Rao}, {and}
  \bibinfo{person}{Minoru Asada}.} \bibinfo{year}{2016}\natexlab{}.
\newblock \showarticletitle{{Initiative in robot assistance during
  collaborative task execution}}. In \bibinfo{booktitle}{\emph{2016 11th
  ACM/IEEE International Conference on Human-Robot Interaction (HRI)}}.
  \bibinfo{publisher}{IEEE}, \bibinfo{pages}{67--74}.
\newblock
\showISBNx{978-1-4673-8370-7}
\urldef\tempurl%
\url{https://doi.org/10.1109/HRI.2016.7451735}
\showDOI{\tempurl}


\bibitem[\protect\citeauthoryear{Barrick, Stewart, Neubert, and Mount}{Barrick
  et~al\mbox{.}}{1998}]%
        {barrick1998relating}
\bibfield{author}{\bibinfo{person}{Murray~R Barrick}, \bibinfo{person}{Greg~L
  Stewart}, \bibinfo{person}{Mitchell~J Neubert}, {and}
  \bibinfo{person}{Michael~K Mount}.} \bibinfo{year}{1998}\natexlab{}.
\newblock \showarticletitle{Relating member ability and personality to
  work-team processes and team effectiveness.}
\newblock \bibinfo{journal}{\emph{Journal of applied psychology}}
  \bibinfo{volume}{83}, \bibinfo{number}{3} (\bibinfo{year}{1998}),
  \bibinfo{pages}{377}.
\newblock


\bibitem[\protect\citeauthoryear{Benton}{Benton}{1971}]%
        {Benton1971ProductivityChildren}
\bibfield{author}{\bibinfo{person}{Alan~A. Benton}.}
  \bibinfo{year}{1971}\natexlab{}.
\newblock \showarticletitle{{Productivity, distributive justice, and bargaining
  among children}}.
\newblock \bibinfo{journal}{\emph{Journal of Personality and Social
  Psychology}} (\bibinfo{year}{1971}).
\newblock
\showISSN{00223514}
\urldef\tempurl%
\url{https://doi.org/10.1037/h0030702}
\showDOI{\tempurl}


\bibitem[\protect\citeauthoryear{Berscheid, Boye, and Walster}{Berscheid
  et~al\mbox{.}}{1968}]%
        {Berscheid1968RetaliationEquity}
\bibfield{author}{\bibinfo{person}{Ellen Berscheid}, \bibinfo{person}{David
  Boye}, {and} \bibinfo{person}{Elaine Walster}.}
  \bibinfo{year}{1968}\natexlab{}.
\newblock \showarticletitle{{Retaliation as a means of restoring equity}}.
\newblock \bibinfo{journal}{\emph{Journal of Personality and Social
  Psychology}} (\bibinfo{year}{1968}).
\newblock
\showISSN{00223514}
\urldef\tempurl%
\url{https://doi.org/10.1037/h0026817}
\showDOI{\tempurl}


\bibitem[\protect\citeauthoryear{Brosnan and de~Waal}{Brosnan and
  de~Waal}{2014}]%
        {brosnan2014evolution}
\bibfield{author}{\bibinfo{person}{Sarah~F Brosnan} {and}
  \bibinfo{person}{Frans~BM de Waal}.} \bibinfo{year}{2014}\natexlab{}.
\newblock \showarticletitle{Evolution of responses to (un) fairness}.
\newblock \bibinfo{journal}{\emph{Science}} \bibinfo{volume}{346},
  \bibinfo{number}{6207} (\bibinfo{year}{2014}), \bibinfo{pages}{1251776}.
\newblock


\bibitem[\protect\citeauthoryear{Burtini, Loeppky, and Lawrence}{Burtini
  et~al\mbox{.}}{2015}]%
        {Burtini2015ABandit}
\bibfield{author}{\bibinfo{person}{Giuseppe Burtini}, \bibinfo{person}{Jason
  Loeppky}, {and} \bibinfo{person}{Ramon Lawrence}.}
  \bibinfo{year}{2015}\natexlab{}.
\newblock \showarticletitle{{A Survey of Online Experiment Design with the
  Stochastic Multi-Armed Bandit}}.
\newblock  (\bibinfo{year}{2015}).
\newblock


\bibitem[\protect\citeauthoryear{Camerer}{Camerer}{2003}]%
        {Camerer2003BehavioralInteraction}
\bibfield{author}{\bibinfo{person}{Colin Camerer}.}
  \bibinfo{year}{2003}\natexlab{}.
\newblock \bibinfo{booktitle}{\emph{{Behavioral game theory : experiments in
  strategic interaction}}}.
\newblock \bibinfo{publisher}{Russell Sage Foundation}. 550 pages.
\newblock
\showISBNx{9780691090399}


\bibitem[\protect\citeauthoryear{Chan, Hadfield-Menell, Srinivasa, and
  Dragan}{Chan et~al\mbox{.}}{2019}]%
        {Chan2019TheBandit}
\bibfield{author}{\bibinfo{person}{Lawrence Chan}, \bibinfo{person}{Dylan
  Hadfield-Menell}, \bibinfo{person}{Siddhartha Srinivasa}, {and}
  \bibinfo{person}{Anca Dragan}.} \bibinfo{year}{2019}\natexlab{}.
\newblock \showarticletitle{{The Assistive Multi-Armed Bandit}}. In
  \bibinfo{booktitle}{\emph{2019 14th ACM/IEEE International Conference on
  Human-Robot Interaction (HRI)}}. \bibinfo{publisher}{IEEE},
  \bibinfo{pages}{354--363}.
\newblock
\showISBNx{978-1-5386-8555-6}
\urldef\tempurl%
\url{https://doi.org/10.1109/HRI.2019.8673234}
\showDOI{\tempurl}


\bibitem[\protect\citeauthoryear{Chen, Nikolaidis, Soh, Hsu, and
  Srinivasa}{Chen et~al\mbox{.}}{2018}]%
        {Chen2018PlanningCollaboration}
\bibfield{author}{\bibinfo{person}{Min Chen}, \bibinfo{person}{Stefanos
  Nikolaidis}, \bibinfo{person}{Harold Soh}, \bibinfo{person}{David Hsu}, {and}
  \bibinfo{person}{Siddhartha Srinivasa}.} \bibinfo{year}{2018}\natexlab{}.
\newblock \showarticletitle{{Planning with Trust for Human-Robot
  Collaboration}}. In \bibinfo{booktitle}{\emph{ACM/IEEE International
  Conference on Human-Robot Interaction}}.
\newblock
\showISBNx{9781450349536}
\showISSN{21672148}
\urldef\tempurl%
\url{https://doi.org/10.1145/3171221.3171264}
\showDOI{\tempurl}


\bibitem[\protect\citeauthoryear{Cheung, Simchi-Levi, and Zhu}{Cheung
  et~al\mbox{.}}{2018}]%
        {cheung2018learning}
\bibfield{author}{\bibinfo{person}{Wang~Chi Cheung}, \bibinfo{person}{David
  Simchi-Levi}, {and} \bibinfo{person}{Ruihao Zhu}.}
  \bibinfo{year}{2018}\natexlab{}.
\newblock \showarticletitle{Learning to optimize under non-stationarity}.
\newblock \bibinfo{journal}{\emph{arXiv preprint arXiv:1810.03024}}
  (\bibinfo{year}{2018}).
\newblock


\bibitem[\protect\citeauthoryear{Cohen-Charash and Spector}{Cohen-Charash and
  Spector}{2001}]%
        {Cohen-Charash2001TheMeta-analysis}
\bibfield{author}{\bibinfo{person}{Yochi Cohen-Charash} {and}
  \bibinfo{person}{Paul~E. Spector}.} \bibinfo{year}{2001}\natexlab{}.
\newblock \showarticletitle{{The role of justice in organizations: A
  meta-analysis}}.
\newblock \bibinfo{journal}{\emph{Organizational Behavior and Human Decision
  Processes}} (\bibinfo{year}{2001}).
\newblock
\showISSN{07495978}
\urldef\tempurl%
\url{https://doi.org/10.1006/obhd.2001.2958}
\showDOI{\tempurl}


\bibitem[\protect\citeauthoryear{Colquitt, Wesson, Porter, Conlon, and
  Ng}{Colquitt et~al\mbox{.}}{2001}]%
        {Colquitt2001JusticeResearch}
\bibfield{author}{\bibinfo{person}{Jason~A. Colquitt},
  \bibinfo{person}{Michael~J. Wesson}, \bibinfo{person}{Christopher~O.L.H.
  Porter}, \bibinfo{person}{Donald~E. Conlon}, {and} \bibinfo{person}{K.~Yee
  Ng}.} \bibinfo{year}{2001}\natexlab{}.
\newblock \bibinfo{title}{{Justice at the millennium: A meta-analytic review of
  25 years of organizational justice research}}.
\newblock
\newblock
\showISSN{00219010}
\urldef\tempurl%
\url{https://doi.org/10.1037/0021-9010.86.3.425}
\showDOI{\tempurl}


\bibitem[\protect\citeauthoryear{Crandall, Oudah, Ishowo-Oloko, Abdallah,
  Bonnefon, Cebrian, Shariff, Goodrich, Rahwan, et~al\mbox{.}}{Crandall
  et~al\mbox{.}}{2018}]%
        {crandall2018cooperating}
\bibfield{author}{\bibinfo{person}{Jacob~W Crandall}, \bibinfo{person}{Mayada
  Oudah}, \bibinfo{person}{Fatimah Ishowo-Oloko}, \bibinfo{person}{Sherief
  Abdallah}, \bibinfo{person}{Jean-Fran{\c{c}}ois Bonnefon},
  \bibinfo{person}{Manuel Cebrian}, \bibinfo{person}{Azim Shariff},
  \bibinfo{person}{Michael~A Goodrich}, \bibinfo{person}{Iyad Rahwan},
  {et~al\mbox{.}}} \bibinfo{year}{2018}\natexlab{}.
\newblock \showarticletitle{Cooperating with machines}.
\newblock \bibinfo{journal}{\emph{Nature communications}} \bibinfo{volume}{9},
  \bibinfo{number}{1} (\bibinfo{year}{2018}), \bibinfo{pages}{233}.
\newblock


\bibitem[\protect\citeauthoryear{Culnan and Armstrong}{Culnan and
  Armstrong}{1999}]%
        {culnan1999information}
\bibfield{author}{\bibinfo{person}{Mary~J Culnan} {and}
  \bibinfo{person}{Pamela~K Armstrong}.} \bibinfo{year}{1999}\natexlab{}.
\newblock \showarticletitle{Information privacy concerns, procedural fairness,
  and impersonal trust: An empirical investigation}.
\newblock \bibinfo{journal}{\emph{Organization science}} \bibinfo{volume}{10},
  \bibinfo{number}{1} (\bibinfo{year}{1999}), \bibinfo{pages}{104--115}.
\newblock


\bibitem[\protect\citeauthoryear{Desai, Kaniarasu, Medvedev, Steinfeld, and
  Yanco}{Desai et~al\mbox{.}}{2013}]%
        {desai2013impact}
\bibfield{author}{\bibinfo{person}{Munjal Desai}, \bibinfo{person}{Poornima
  Kaniarasu}, \bibinfo{person}{Mikhail Medvedev}, \bibinfo{person}{Aaron
  Steinfeld}, {and} \bibinfo{person}{Holly Yanco}.}
  \bibinfo{year}{2013}\natexlab{}.
\newblock \showarticletitle{Impact of robot failures and feedback on real-time
  trust}. In \bibinfo{booktitle}{\emph{Proceedings of the 8th ACM/IEEE
  international conference on Human-robot interaction}}. IEEE Press,
  \bibinfo{pages}{251--258}.
\newblock


\bibitem[\protect\citeauthoryear{Desai, Medvedev, V{\'a}zquez, McSheehy,
  Gadea-Omelchenko, Bruggeman, Steinfeld, and Yanco}{Desai
  et~al\mbox{.}}{2012}]%
        {desai2012effects}
\bibfield{author}{\bibinfo{person}{Munjal Desai}, \bibinfo{person}{Mikhail
  Medvedev}, \bibinfo{person}{Marynel V{\'a}zquez}, \bibinfo{person}{Sean
  McSheehy}, \bibinfo{person}{Sofia Gadea-Omelchenko},
  \bibinfo{person}{Christian Bruggeman}, \bibinfo{person}{Aaron Steinfeld},
  {and} \bibinfo{person}{Holly Yanco}.} \bibinfo{year}{2012}\natexlab{}.
\newblock \showarticletitle{Effects of changing reliability on trust of robot
  systems}. In \bibinfo{booktitle}{\emph{Proceedings of the seventh annual
  ACM/IEEE international conference on Human-Robot Interaction}}. ACM,
  \bibinfo{pages}{73--80}.
\newblock


\bibitem[\protect\citeauthoryear{Dudik, Hsu, Kale, Karampatziakis, Langford,
  Reyzin, and Zhang}{Dudik et~al\mbox{.}}{[n. d.]}]%
        {DudikEfficientBandits}
\bibfield{author}{\bibinfo{person}{Miroslav Dudik}, \bibinfo{person}{Daniel
  Hsu}, \bibinfo{person}{Satyen Kale}, \bibinfo{person}{Nikos Karampatziakis},
  \bibinfo{person}{John Langford}, \bibinfo{person}{Lev Reyzin}, {and}
  \bibinfo{person}{Tong Zhang}.} \bibinfo{year}{[n. d.]}\natexlab{}.
\newblock \bibinfo{booktitle}{\emph{{Efficient Optimal Learning for Contextual
  Bandits}}}.
\newblock \bibinfo{type}{{T}echnical {R}eport}.
\newblock


\bibitem[\protect\citeauthoryear{Fisman, Kariv, and Markovits}{Fisman
  et~al\mbox{.}}{2007}]%
        {FismanWeGiving}
\bibfield{author}{\bibinfo{person}{Raymond Fisman}, \bibinfo{person}{Shachar
  Kariv}, {and} \bibinfo{person}{Daniel Markovits}.}
  \bibinfo{year}{2007}\natexlab{}.
\newblock \showarticletitle{Individual preferences for giving}.
\newblock \bibinfo{journal}{\emph{American Economic Review}}
  \bibinfo{volume}{97}, \bibinfo{number}{5} (\bibinfo{year}{2007}),
  \bibinfo{pages}{1858--1876}.
\newblock


\bibitem[\protect\citeauthoryear{Garivier}{Garivier}{2008}]%
        {Garivier2008OnMoulines}
\bibfield{author}{\bibinfo{person}{Aurelien Garivier}.}
  \bibinfo{year}{2008}\natexlab{}.
\newblock \bibinfo{booktitle}{\emph{{On Upper-Confidence Bound Policies for
  Non-Stationary Bandit Problems Eric Moulines}}}.
\newblock \bibinfo{type}{{T}echnical {R}eport}.
\newblock


\bibitem[\protect\citeauthoryear{Garivier and Moulines}{Garivier and
  Moulines}{2011}]%
        {garivier2011upper}
\bibfield{author}{\bibinfo{person}{Aurelien Garivier} {and}
  \bibinfo{person}{Eric Moulines}.} \bibinfo{year}{2011}\natexlab{}.
\newblock \showarticletitle{On upper-confidence bound policies for switching
  bandit problems}. In \bibinfo{booktitle}{\emph{International Conference on
  Algorithmic Learning Theory}}. Springer, \bibinfo{pages}{174--188}.
\newblock


\bibitem[\protect\citeauthoryear{Gombolay, Gutierrez, Clarke, Sturla, and
  Shah}{Gombolay et~al\mbox{.}}{2015a}]%
        {gombolay2015decision}
\bibfield{author}{\bibinfo{person}{Matthew~C Gombolay},
  \bibinfo{person}{Reymundo~A Gutierrez}, \bibinfo{person}{Shanelle~G Clarke},
  \bibinfo{person}{Giancarlo~F Sturla}, {and} \bibinfo{person}{Julie~A Shah}.}
  \bibinfo{year}{2015}\natexlab{a}.
\newblock \showarticletitle{Decision-making authority, team efficiency and
  human worker satisfaction in mixed human--robot teams}.
\newblock \bibinfo{journal}{\emph{Autonomous Robots}} \bibinfo{volume}{39},
  \bibinfo{number}{3} (\bibinfo{year}{2015}), \bibinfo{pages}{293--312}.
\newblock


\bibitem[\protect\citeauthoryear{Gombolay, Huang, and Shah}{Gombolay
  et~al\mbox{.}}{2015b}]%
        {Gombolay2015CoordinationPreferences}
\bibfield{author}{\bibinfo{person}{Matthew~C Gombolay}, \bibinfo{person}{Cindy
  Huang}, {and} \bibinfo{person}{Julie~A Shah}.}
  \bibinfo{year}{2015}\natexlab{b}.
\newblock \showarticletitle{{Coordination of Human-Robot Teaming with Human
  Task Preferences}}.
\newblock \bibinfo{journal}{\emph{AAAI Fall Symposium Series on AI-HRI}}
  (\bibinfo{year}{2015}).
\newblock
\showISBNx{9781577357476}


\bibitem[\protect\citeauthoryear{Graf and Green}{Graf and Green}{1971}]%
        {Graf1971TheRetaliation}
\bibfield{author}{\bibinfo{person}{Richard~G. Graf} {and}
  \bibinfo{person}{Duane Green}.} \bibinfo{year}{1971}\natexlab{}.
\newblock \showarticletitle{{The equity restoring components of retaliation}}.
\newblock \bibinfo{journal}{\emph{Journal of Personality}}
  (\bibinfo{year}{1971}).
\newblock
\showISSN{14676494}
\urldef\tempurl%
\url{https://doi.org/10.1111/j.1467-6494.1971.tb00064.x}
\showDOI{\tempurl}


\bibitem[\protect\citeauthoryear{Groom and Nass}{Groom and Nass}{2007}]%
        {Groom2007CanInteraction}
\bibfield{author}{\bibinfo{person}{Victoria Groom} {and}
  \bibinfo{person}{Clifford Nass}.} \bibinfo{year}{2007}\natexlab{}.
\newblock \showarticletitle{Can robots be teammates?: Benchmarks in human-robot
  teams}.
\newblock \bibinfo{journal}{\emph{Interaction Studies}} \bibinfo{volume}{8},
  \bibinfo{number}{3} (\bibinfo{year}{2007}), \bibinfo{pages}{483--500}.
\newblock


\bibitem[\protect\citeauthoryear{Hackman and Oldham}{Hackman and
  Oldham}{1976}]%
        {hackman1976motivation}
\bibfield{author}{\bibinfo{person}{J~Richard Hackman} {and}
  \bibinfo{person}{Greg~R Oldham}.} \bibinfo{year}{1976}\natexlab{}.
\newblock \showarticletitle{Motivation through the design of work: Test of a
  theory}.
\newblock \bibinfo{journal}{\emph{Organizational behavior and human
  performance}} \bibinfo{volume}{16}, \bibinfo{number}{2}
  (\bibinfo{year}{1976}), \bibinfo{pages}{250--279}.
\newblock


\bibitem[\protect\citeauthoryear{Haier, Siegel~Jr, MacLachlan, Soderling,
  Lottenberg, and Buchsbaum}{Haier et~al\mbox{.}}{1992}]%
        {haier1992regional}
\bibfield{author}{\bibinfo{person}{Richard~J Haier},
  \bibinfo{person}{Benjamin~V Siegel~Jr}, \bibinfo{person}{Andrew MacLachlan},
  \bibinfo{person}{Eric Soderling}, \bibinfo{person}{Stephen Lottenberg}, {and}
  \bibinfo{person}{Monte~S Buchsbaum}.} \bibinfo{year}{1992}\natexlab{}.
\newblock \showarticletitle{Regional glucose metabolic changes after learning a
  complex visuospatial/motor task: a positron emission tomographic study}.
\newblock \bibinfo{journal}{\emph{Brain research}} \bibinfo{volume}{570},
  \bibinfo{number}{1-2} (\bibinfo{year}{1992}), \bibinfo{pages}{134--143}.
\newblock


\bibitem[\protect\citeauthoryear{Hayes and Scassellati}{Hayes and
  Scassellati}{2015}]%
        {hayes2015effective}
\bibfield{author}{\bibinfo{person}{Bradley Hayes} {and} \bibinfo{person}{Brian
  Scassellati}.} \bibinfo{year}{2015}\natexlab{}.
\newblock \showarticletitle{Effective robot teammate behaviors for supporting
  sequential manipulation tasks}. In \bibinfo{booktitle}{\emph{2015 IEEE/RSJ
  International Conference on Intelligent Robots and Systems (IROS)}}. IEEE,
  \bibinfo{pages}{6374--6380}.
\newblock


\bibitem[\protect\citeauthoryear{Huang, Cakmak, and Mutlu}{Huang
  et~al\mbox{.}}{[n. d.]}]%
        {HuangAdaptiveHandovers}
\bibfield{author}{\bibinfo{person}{Chien-Ming Huang}, \bibinfo{person}{Maya
  Cakmak}, {and} \bibinfo{person}{Bilge Mutlu}.} \bibinfo{year}{[n.
  d.]}\natexlab{}.
\newblock \showarticletitle{{Adaptive Coordination Strategies for Human-Robot
  Handovers}}.
\newblock  (\bibinfo{year}{[n. d.]}).
\newblock


\bibitem[\protect\citeauthoryear{Joseph, Kearns, Morgenstern, and Roth}{Joseph
  et~al\mbox{.}}{2016}]%
        {joseph2016fairness}
\bibfield{author}{\bibinfo{person}{Matthew Joseph}, \bibinfo{person}{Michael
  Kearns}, \bibinfo{person}{Jamie~H Morgenstern}, {and} \bibinfo{person}{Aaron
  Roth}.} \bibinfo{year}{2016}\natexlab{}.
\newblock \showarticletitle{Fairness in learning: Classic and contextual
  bandits}. In \bibinfo{booktitle}{\emph{Advances in Neural Information
  Processing Systems}}. \bibinfo{pages}{325--333}.
\newblock


\bibitem[\protect\citeauthoryear{Jung, DiFranzo, Stoll, Shen, Lawrence, and
  Claure}{Jung et~al\mbox{.}}{2018}]%
        {jung2018robot}
\bibfield{author}{\bibinfo{person}{Malte~F Jung}, \bibinfo{person}{Dominic
  DiFranzo}, \bibinfo{person}{Brett Stoll}, \bibinfo{person}{Solace Shen},
  \bibinfo{person}{Austin Lawrence}, {and} \bibinfo{person}{Houston Claure}.}
  \bibinfo{year}{2018}\natexlab{}.
\newblock \showarticletitle{Robot Assisted Tower Construction-A Resource
  Distribution Task to Study Human-Robot Collaboration and Interaction with
  Groups of People}.
\newblock \bibinfo{journal}{\emph{arXiv preprint arXiv:1812.09548}}
  (\bibinfo{year}{2018}).
\newblock


\bibitem[\protect\citeauthoryear{Kaniarasu, Steinfeld, Desai, and
  Yanco}{Kaniarasu et~al\mbox{.}}{2012}]%
        {kaniarasu2012potential}
\bibfield{author}{\bibinfo{person}{Poornima Kaniarasu}, \bibinfo{person}{Aaron
  Steinfeld}, \bibinfo{person}{Munjal Desai}, {and} \bibinfo{person}{Holly
  Yanco}.} \bibinfo{year}{2012}\natexlab{}.
\newblock \showarticletitle{Potential measures for detecting trust changes}. In
  \bibinfo{booktitle}{\emph{Proceedings of the seventh annual ACM/IEEE
  international conference on Human-Robot Interaction}}. ACM,
  \bibinfo{pages}{241--242}.
\newblock


\bibitem[\protect\citeauthoryear{Kirsh and Maglio}{Kirsh and Maglio}{1994}]%
        {kirsh1994distinguishing}
\bibfield{author}{\bibinfo{person}{David Kirsh} {and} \bibinfo{person}{Paul
  Maglio}.} \bibinfo{year}{1994}\natexlab{}.
\newblock \showarticletitle{On distinguishing epistemic from pragmatic action}.
\newblock \bibinfo{journal}{\emph{Cognitive science}} \bibinfo{volume}{18},
  \bibinfo{number}{4} (\bibinfo{year}{1994}), \bibinfo{pages}{513--549}.
\newblock


\bibitem[\protect\citeauthoryear{Kleinberg, Slivkins, and Upfal}{Kleinberg
  et~al\mbox{.}}{2008}]%
        {KleinbergMulti-ArmedSpaces}
\bibfield{author}{\bibinfo{person}{Robert Kleinberg},
  \bibinfo{person}{Aleksandrs Slivkins}, {and} \bibinfo{person}{Eli Upfal}.}
  \bibinfo{year}{2008}\natexlab{}.
\newblock \showarticletitle{Multi-armed bandits in metric spaces}. In
  \bibinfo{booktitle}{\emph{Proceedings of the fortieth annual ACM symposium on
  Theory of computing}}. ACM, \bibinfo{pages}{681--690}.
\newblock


\bibitem[\protect\citeauthoryear{Lai~Andherbertrobbins}{Lai~Andherbertrobbins}{1985}]%
        {LaiAndherbertrobbins1985AsymptoticallyRules}
\bibfield{author}{\bibinfo{person}{T~L Lai~Andherbertrobbins}.}
  \bibinfo{year}{1985}\natexlab{}.
\newblock \showarticletitle{{Asymptotically Efficient Adaptive Allocation
  Rules*}}.
\newblock   \bibinfo{volume}{6} (\bibinfo{year}{1985}), \bibinfo{pages}{4--22}.
\newblock


\bibitem[\protect\citeauthoryear{Lan, Kao, Chiang, and Sabharwal}{Lan
  et~al\mbox{.}}{2010}]%
        {lan2010axiomatic}
\bibfield{author}{\bibinfo{person}{Tian Lan}, \bibinfo{person}{David Kao},
  \bibinfo{person}{Mung Chiang}, {and} \bibinfo{person}{Ashutosh Sabharwal}.}
  \bibinfo{year}{2010}\natexlab{}.
\newblock \bibinfo{booktitle}{\emph{An axiomatic theory of fairness in network
  resource allocation}}.
\newblock \bibinfo{publisher}{IEEE}.
\newblock


\bibitem[\protect\citeauthoryear{Lane and Messe}{Lane and Messe}{1971}]%
        {Lane1971EquityRewards}
\bibfield{author}{\bibinfo{person}{Irving~M. Lane} {and}
  \bibinfo{person}{Lawrence~A. Messe}.} \bibinfo{year}{1971}\natexlab{}.
\newblock \showarticletitle{{Equity and the distribution of rewards}}.
\newblock \bibinfo{journal}{\emph{Journal of Personality and Social
  Psychology}} (\bibinfo{year}{1971}).
\newblock
\showISSN{00223514}
\urldef\tempurl%
\url{https://doi.org/10.1037/h0031684}
\showDOI{\tempurl}


\bibitem[\protect\citeauthoryear{Lange, Batson, De~Bruin, Koole, and
  Lange}{Lange et~al\mbox{.}}{1999}]%
        {Lange1999TheOrientation}
\bibfield{author}{\bibinfo{person}{Paul A M~Van Lange}, \bibinfo{person}{Daniel
  Batson}, \bibinfo{person}{Ellen De~Bruin}, \bibinfo{person}{Sander Koole},
  {and} \bibinfo{person}{A~M~Van Lange}.} \bibinfo{year}{1999}\natexlab{}.
\newblock \showarticletitle{{The Pursuit of Joint Outcomes and Equality in
  Outcomes: An Integrative Model of Social Value Orientation}}.
\newblock  \bibinfo{volume}{77}, \bibinfo{number}{2} (\bibinfo{year}{1999}),
  \bibinfo{pages}{337--349}.
\newblock


\bibitem[\protect\citeauthoryear{Lecture, Fehr, and Falk}{Lecture
  et~al\mbox{.}}{2002}]%
        {Lecture2002PsychologicalIncentives}
\bibfield{author}{\bibinfo{person}{Joseph~Schumpeter Lecture},
  \bibinfo{person}{Ernst Fehr}, {and} \bibinfo{person}{Armin Falk}.}
  \bibinfo{year}{2002}\natexlab{}.
\newblock \showarticletitle{{Psychological foundations of incentives}}.
\newblock   \bibinfo{volume}{46} (\bibinfo{year}{2002}),
  \bibinfo{pages}{687--724}.
\newblock


\bibitem[\protect\citeauthoryear{Lee}{Lee}{2018}]%
        {lee2018understanding}
\bibfield{author}{\bibinfo{person}{Min~Kyung Lee}.}
  \bibinfo{year}{2018}\natexlab{}.
\newblock \showarticletitle{Understanding perception of algorithmic decisions:
  Fairness, trust, and emotion in response to algorithmic management}.
\newblock \bibinfo{journal}{\emph{Big Data \& Society}} \bibinfo{volume}{5},
  \bibinfo{number}{1} (\bibinfo{year}{2018}),
  \bibinfo{pages}{2053951718756684}.
\newblock


\bibitem[\protect\citeauthoryear{Lee and Baykal}{Lee and Baykal}{[n. d.]}]%
        {LeeAlgorithmicDivision}
\bibfield{author}{\bibinfo{person}{Min~Kyung Lee} {and} \bibinfo{person}{Su
  Baykal}.} \bibinfo{year}{[n. d.]}\natexlab{}.
\newblock \showarticletitle{{Algorithmic Mediation in Group Decisions: Fairness
  Perceptions of Algorithmically Mediated vs. Discussion-Based Social
  Division}}.
\newblock  (\bibinfo{year}{[n. d.]}).
\newblock
\showISBNx{9781450343350}
\urldef\tempurl%
\url{https://doi.org/10.1145/2998181.2998230}
\showDOI{\tempurl}


\bibitem[\protect\citeauthoryear{Li, Liu, and Ji}{Li et~al\mbox{.}}{2019}]%
        {li2019combinatorial}
\bibfield{author}{\bibinfo{person}{Fengjiao Li}, \bibinfo{person}{Jia Liu},
  {and} \bibinfo{person}{Bo Ji}.} \bibinfo{year}{2019}\natexlab{}.
\newblock \showarticletitle{Combinatorial Sleeping Bandits with Fairness
  Constraints}.
\newblock \bibinfo{journal}{\emph{arXiv preprint arXiv:1901.04891}}
  (\bibinfo{year}{2019}).
\newblock


\bibitem[\protect\citeauthoryear{Li, Chu, Langford, and Schapire}{Li
  et~al\mbox{.}}{2012}]%
        {Li2012ARecommendation}
\bibfield{author}{\bibinfo{person}{Lihong Li}, \bibinfo{person}{Wei Chu},
  \bibinfo{person}{John Langford}, {and} \bibinfo{person}{Robert~E Schapire}.}
  \bibinfo{year}{2012}\natexlab{}.
\newblock \showarticletitle{{A Contextual-Bandit Approach to Personalized News
  Article Recommendation}}.
\newblock  (\bibinfo{year}{2012}).
\newblock
\urldef\tempurl%
\url{https://arxiv.org/pdf/1003.0146.pdf}
\showURL{%
\tempurl}


\bibitem[\protect\citeauthoryear{Lindstedt and Gray}{Lindstedt and
  Gray}{2013}]%
        {lindstedt2013extreme}
\bibfield{author}{\bibinfo{person}{John Lindstedt} {and} \bibinfo{person}{Wayne
  Gray}.} \bibinfo{year}{2013}\natexlab{}.
\newblock \showarticletitle{Extreme expertise: Exploring expert behavior in
  Tetris}. In \bibinfo{booktitle}{\emph{Proceedings of the Annual Meeting of
  the Cognitive Science Society}}, Vol.~\bibinfo{volume}{35}.
\newblock


\bibitem[\protect\citeauthoryear{Maillard, Munos, and Stoltz}{Maillard
  et~al\mbox{.}}{2011}]%
        {Maillard2011ADivergences}
\bibfield{author}{\bibinfo{person}{Odalric-Ambrym Maillard},
  \bibinfo{person}{Remi Munos}, {and} \bibinfo{person}{Gilles Stoltz}.}
  \bibinfo{year}{2011}\natexlab{}.
\newblock \bibinfo{booktitle}{\emph{{A Finite-Time Analysis of Multi-armed
  Bandits Problems with Kullback-Leibler Divergences}}}.
\newblock \bibinfo{type}{{T}echnical {R}eport}. \bibinfo{pages}{497--514}
  pages.
\newblock


\bibitem[\protect\citeauthoryear{Mcfarlin and Sweeney}{Mcfarlin and
  Sweeney}{1992}]%
        {Mcfarlin1992DistributiveOutcomes}
\bibfield{author}{\bibinfo{person}{Dean~B Mcfarlin} {and}
  \bibinfo{person}{Paul~D Sweeney}.} \bibinfo{year}{1992}\natexlab{}.
\newblock \showarticletitle{{Distributive and Procedural Justice as Predictors
  of Satisfaction with Personal and Organizational Outcomes}}.
\newblock  \bibinfo{volume}{35}, \bibinfo{number}{3} (\bibinfo{year}{1992}),
  \bibinfo{pages}{626--637}.
\newblock


\bibitem[\protect\citeauthoryear{Pandya, Huang, Hadfield-Menell, and
  Dragan}{Pandya et~al\mbox{.}}{[n. d.]}]%
        {PandyaHuman-AIBandits}
\bibfield{author}{\bibinfo{person}{Ravi Pandya}, \bibinfo{person}{Sandy~H
  Huang}, \bibinfo{person}{Dylan Hadfield-Menell}, {and}
  \bibinfo{person}{Anca~D Dragan}.} \bibinfo{year}{[n. d.]}\natexlab{}.
\newblock \showarticletitle{{Human-AI Learning Performance in Multi-Armed
  Bandits}}.
\newblock  (\bibinfo{year}{[n. d.]}).
\newblock


\bibitem[\protect\citeauthoryear{Patil, Ghalme, Nair, and Narahari}{Patil
  et~al\mbox{.}}{2019}]%
        {patil2019achieving}
\bibfield{author}{\bibinfo{person}{Vishakha Patil}, \bibinfo{person}{Ganesh
  Ghalme}, \bibinfo{person}{Vineet Nair}, {and} \bibinfo{person}{Y Narahari}.}
  \bibinfo{year}{2019}\natexlab{}.
\newblock \showarticletitle{Achieving Fairness in the Stochastic Multi-armed
  Bandit Problem}.
\newblock \bibinfo{journal}{\emph{arXiv preprint arXiv:1907.10516}}
  (\bibinfo{year}{2019}).
\newblock


\bibitem[\protect\citeauthoryear{Shah, Wiken, Williams, and Breazeal}{Shah
  et~al\mbox{.}}{2011a}]%
        {Shah2011ImprovedSystem}
\bibfield{author}{\bibinfo{person}{Julie Shah}, \bibinfo{person}{James Wiken},
  \bibinfo{person}{Brian Williams}, {and} \bibinfo{person}{Cynthia Breazeal}.}
  \bibinfo{year}{2011}\natexlab{a}.
\newblock \showarticletitle{{Improved human-robot team performance using
  chaski, a human-inspired plan execution system}}. In
  \bibinfo{booktitle}{\emph{Proceedings of the 6th international conference on
  Human-robot interaction - HRI '11}}.
\newblock
\showISBNx{9781450305617}
\showISSN{2167-2121}
\urldef\tempurl%
\url{https://doi.org/10.1145/1957656.1957668}
\showDOI{\tempurl}


\bibitem[\protect\citeauthoryear{Shah, Wiken, Williams, and Breazeal}{Shah
  et~al\mbox{.}}{2011b}]%
        {shah2011improved}
\bibfield{author}{\bibinfo{person}{Julie Shah}, \bibinfo{person}{James Wiken},
  \bibinfo{person}{Brian Williams}, {and} \bibinfo{person}{Cynthia Breazeal}.}
  \bibinfo{year}{2011}\natexlab{b}.
\newblock \showarticletitle{Improved human-robot team performance using chaski,
  a human-inspired plan execution system}. In
  \bibinfo{booktitle}{\emph{Proceedings of the 6th international conference on
  Human-robot interaction}}. ACM, \bibinfo{pages}{29--36}.
\newblock


\bibitem[\protect\citeauthoryear{Short and Mataric}{Short and Mataric}{2017}]%
        {short2017robot}
\bibfield{author}{\bibinfo{person}{Elaine Short} {and} \bibinfo{person}{Maja~J
  Mataric}.} \bibinfo{year}{2017}\natexlab{}.
\newblock \showarticletitle{Robot moderation of a collaborative game: Towards
  socially assistive robotics in group interactions}. In
  \bibinfo{booktitle}{\emph{2017 26th IEEE International Symposium on Robot and
  Human Interactive Communication (RO-MAN)}}. IEEE, \bibinfo{pages}{385--390}.
\newblock


\bibitem[\protect\citeauthoryear{Shu, Min, Bodala, Nikolaidis, Hsu, and
  Soh}{Shu et~al\mbox{.}}{2018}]%
        {Shu2018HumanTasks}
\bibfield{author}{\bibinfo{person}{Pan Shu}, \bibinfo{person}{Chen Min},
  \bibinfo{person}{Indu Bodala}, \bibinfo{person}{Stefanos Nikolaidis},
  \bibinfo{person}{David Hsu}, {and} \bibinfo{person}{Harold Soh}.}
  \bibinfo{year}{2018}\natexlab{}.
\newblock \showarticletitle{{Human Trust in Robot Capabilities across Tasks}}.
  In \bibinfo{booktitle}{\emph{ACM/IEEE International Conference on Human-Robot
  Interaction}}.
\newblock
\showISBNx{9781450356152}
\showISSN{21672148}
\urldef\tempurl%
\url{https://doi.org/10.1145/3173386.3177034}
\showDOI{\tempurl}


\bibitem[\protect\citeauthoryear{Skarlicki and Folger}{Skarlicki and
  Folger}{1997}]%
        {skarlicki1997retaliation}
\bibfield{author}{\bibinfo{person}{Daniel~P Skarlicki} {and}
  \bibinfo{person}{Robert Folger}.} \bibinfo{year}{1997}\natexlab{}.
\newblock \showarticletitle{Retaliation in the workplace: The roles of
  distributive, procedural, and interactional justice.}
\newblock \bibinfo{journal}{\emph{Journal of applied Psychology}}
  \bibinfo{volume}{82}, \bibinfo{number}{3} (\bibinfo{year}{1997}),
  \bibinfo{pages}{434}.
\newblock


\bibitem[\protect\citeauthoryear{Stacy}{Stacy}{1963}]%
        {Stacy1963TowardInequity}
\bibfield{author}{\bibinfo{person}{AdamsJ. Stacy}.}
  \bibinfo{year}{1963}\natexlab{}.
\newblock \showarticletitle{{Toward and Understanding of Inequity}}.
\newblock \bibinfo{journal}{\emph{Journal of Abnormal Psychology}}
  (\bibinfo{year}{1963}).
\newblock


\bibitem[\protect\citeauthoryear{V{\'a}zquez, Steinfeld, and
  Hudson}{V{\'a}zquez et~al\mbox{.}}{2016}]%
        {vazquez2016maintaining}
\bibfield{author}{\bibinfo{person}{Marynel V{\'a}zquez}, \bibinfo{person}{Aaron
  Steinfeld}, {and} \bibinfo{person}{Scott~E Hudson}.}
  \bibinfo{year}{2016}\natexlab{}.
\newblock \showarticletitle{Maintaining awareness of the focus of attention of
  a conversation: A robot-centric reinforcement learning approach}. In
  \bibinfo{booktitle}{\emph{2016 25th IEEE International Symposium on Robot and
  Human Interactive Communication (RO-MAN)}}. IEEE, \bibinfo{pages}{36--43}.
\newblock


\bibitem[\protect\citeauthoryear{Walster, Berscheid, and Walster}{Walster
  et~al\mbox{.}}{1973}]%
        {Walster1973NewResearch}
\bibfield{author}{\bibinfo{person}{Elaine Walster}, \bibinfo{person}{Ellen
  Berscheid}, {and} \bibinfo{person}{G.~William Walster}.}
  \bibinfo{year}{1973}\natexlab{}.
\newblock \showarticletitle{{New directions in equity research}}.
\newblock \bibinfo{journal}{\emph{Journal of Personality and Social
  Psychology}} (\bibinfo{year}{1973}).
\newblock
\showISSN{00223514}
\urldef\tempurl%
\url{https://doi.org/10.1037/h0033967}
\showDOI{\tempurl}


\bibitem[\protect\citeauthoryear{Wei and Srivatsva}{Wei and Srivatsva}{2018}]%
        {wei2018abruptly}
\bibfield{author}{\bibinfo{person}{Lai Wei} {and} \bibinfo{person}{Vaibhav
  Srivatsva}.} \bibinfo{year}{2018}\natexlab{}.
\newblock \showarticletitle{On abruptly-changing and slowly-varying multiarmed
  bandit problems}. In \bibinfo{booktitle}{\emph{2018 Annual American Control
  Conference (ACC)}}. IEEE, \bibinfo{pages}{6291--6296}.
\newblock


\end{thebibliography}

\appendix
\section*{Appendix}

\subsection{Notations}

Some notations have already been defined in the main section, but for clarity, we will restate here: Denote $\Delta_i = \mu(i^*)-\mu(i)$. Let $\hat{\mu}_t(i) = \frac{1}{t-1}\sum_{s=1}^{t-1} r_t(i)$ be the mean of empirical rewards for arm $i$ at time $t$ so far and let $n_{t-1}(i)$ be the total number of times arm $i$ has been pulled before time $t$. So $UCB_t(i) = \hat{\mu}_t(i) + \sqrt{\frac{\ln T}{n_{t-1}(i)}}$ and $i_t = \argmax_{i \in [K]} UCB_t(i)$.

\subsection{Auxiliary Theorems and Lemmas}

\begin{theorem} [Hoeffding's Inequality] 
Let $X_1,...,X_T \in [-B,B]$ for some $B>0$ be independent random variables such that $\E[X_t]=0, \forall t\in[T]$, then we have for all $\delta \in (0,1)$,
\begin{align*}
    Pr(\frac{1}{T}\sum_{t=1}^T X_t \geq B\sqrt{\frac{2\ln{\frac{1}{\delta}}}{T}})\leq \delta
\end{align*}
\end{theorem}

\begin{lemma}
\label{lem:random hoeff 1}
For all arm $i$, if the possible value range of $n_{t-1}(i)$ is $[k_s,k_e]$, then 
\begin{align*}
    Prob\left[\mu(i) - \hat{\mu}_t(i) \geq 2\sqrt{\frac{\ln T}{n_{t-1}(i)}} \right] \leq \sum_{k=k_s}^{k_e} \frac{1}{T^2}
\end{align*}

\end{lemma}
\begin{proof}
We want to bound them by Hoeffding's Inequality, however, one trap here is that $n_{t-1}(i)$ is actually a random variable depending on all the rewards decided by the environment. To deal with this issue, imagine there is a infinite sequence of $X_1(i),X_2(i)...$ of independent samples of $\calD_i$ for each action $i$ and at time t observed reward $r_t(i_t)$ is the $n_t(i_t)$-th sample of this sequence, that is, $r_t(i_t) = X_{n_t(i_t)}(i_t)$. So $\hat{\mu}_{t-1}(i)$ as be written as $\tilde{\mu}_{n_{t-1}(i)}(i) = \frac{1}{n_{t-1}(i)}\sum_{k=1}^{n_{t-1}(i)}X_k(i)$.

So now we want to know what is the possible value of $n_{t-1}(i)$. According to the assumption $n_{t-1}(i) \in [k_s,k_e]$, we have,
\begin{align*}
    &Prob\left[\mu(i) - \hat{\mu}_{t-1}(i) \geq 2\sqrt{\frac{\ln T}{n_{t-1}(i)}} \right]\\
    &\leq Prob\left[\exists k \in [k_s, k_e] \quad s.t. \mu(i) - \tilde{\mu}_{ k}(i) \geq 2\sqrt{\frac{\ln T}{k}} \right]\\
    &\leq \sum_{k=k_s}^{k_e } Prob\left[\mu(i) - \tilde{\mu}_{k}(i) \geq 2\sqrt{\frac{\ln T}{k}} \right]\\
    & \leq \sum_{k=k_s}^{k_e} \frac{1}{T^2} 
\end{align*}
The penultimate inequality is by hoeffding's inequality.
\end{proof}

\subsection{Proof for Algorithm 1}

\subsubsection{Notations}
 We define $\calI$ as the set of "non-prescheduled" time slots among $T$. Let $m_{t-1}(i)$ be the total number of times arm $i$ has been pulled before time $t$ and among $\calI$, so $m_{t-1}(i) \leq n_{t-1}(i)$. Also $\calI[i]$ means the $i$-th time slot in $\calI$.

\subsubsection{Main Proof}
\textbf{First we rewrite this regret in the form of variable $\Delta_i$ and $m_T$},
\begin{align*}
\E \left[ \sum_{t\in \calI} r_t(i^*) - r_t(i_t) \right] 
& = \E \left[   \sum_{t\in  \calI} \mu^* - \mu(i_t) \right]\\
& = \E \left[   \sum_{t\in  \calI} \Delta_{i_t} \right]\\
& = \sum_{i \neq i^*} \Delta_i \E \left[m_T(i) \right]
\end{align*}
Here the first expectation is regarding to the whole environment randomness through $T$. The first equality comes from $\E_{\text{env at t}}(r_t) = \mu$.

\textbf{Next we want to bound $\E\left[m_T\right] $ following the similar idea as in the original UCB paper.} 
\begin{align*}
\E\left[m_T(i)\right] 
&= m + \sum_{t \in \calI, t> \calI[m]} Prob\left[ (i_t = i)\quad  and \quad m_{t-1} \geq m \right] \\
& \leq m + \sum_{t \in \calI, t> \calI[m]}  \underbrace{Prob \left[UCB_{t}(i)>UCB_{t}(i^{*})\quad and \quad m_{t-1} \geq m  \right]}_{\textsc{Term1}}
\end{align*}
Here $m$ can be any non-negative integer. In the later analysis, choice of $m$ helps us to get a tighter bound.

\textbf{Now we analyze the \textsc{Term1}}.

\begin{align*}
    \textsc{Term1}
    &\leq Prob\left[ UCB_{t}(i^{*}) < \mu(i^{*}) \right] + Prob\left[ UCB_{t}(i) > \mu(i^{*}) \quad and \quad  m_{t-1}>m \right] \\
    & \leq Prob\left[\mu(i^{*}) - \hat{\mu}_t(i^{*}) \geq 2\sqrt{\frac{\ln T}{n_{t-1}(i^{*})}} \right] \\
    & \quad + Prob\left[ \hat{\mu}_t(i) -  \mu(i) \geq \Delta_i - 2\sqrt{\frac{\ln T}{n_{t-1}(i)}} \quad and \quad  m_{t-1}>m \right] \\   
\end{align*}
First, observe that $Prob\left[\mu(i^{*}) - \hat{\mu}_t(i^{*}) \geq 2\sqrt{\frac{\ln T}{n_{t-1}(i^{*})}} \right]$ has nothing to do with $m$, we can directly apply Lemma~\ref{lem:random hoeff 1} to get upper the bound. So now we want to know what is the $[k_s,k_e]$. First, because we played each arm once at the beginning, $n_{t-1}(i)$ should at least be $1$. Then, because at time $t$ there will $\lfloor (t-1-K)v\rfloor$ blocks and in each block we pull each arm at least once due to pre-scheduling, so  $k_s = \lfloor (t-1-K)v\rfloor + 1$.  Finally, because each arm will have been pulled at least $k_s$ times, so $k_e = t-1- (K-1)k_s$. So the upper bound is 
\begin{align*}
     \sum_{k=k_s}^{t-(K-1)k_s-1} \frac{1}{T^2} & \leq \sum_{k=1}^{T-K\lfloor (T-1-K)v\rfloor}\frac{1}{T^2}\\  
     & \leq \frac{(1-Kv)}{T} + \frac{Kv+K^2v + 1}{T^2} \leq \frac{(1-Kv)}{T} + \order(\frac{K}{T^2}) 
\end{align*}

Then we are going to deal with $Prob\left[ \hat{\mu}_t(i) -  \mu(i) \geq \Delta_i - 2\sqrt{\frac{\ln T}{n_{t-1}(i)}} \quad and \quad  m_{t-1}>m \right]$. Again we want to use the Lemma~\ref{lem:random hoeff 1}, but we need to choose $m$ at first. The reason we want to choose $m$ is that we consider that in the first $m$ epochs the bound will be very loose, so we can directly bound the probability by $1$.

Note that we can easily make connections between $n_{t-1}(i)$ and $m$,
\begin{align*}
    n_{t-1}(i)  
    & \geq \lfloor v(t-1-K) \rfloor + m_{t-1} + 1  \\
    & \geq (m_{t-1}-\frac{1}{v} + K)*\frac{1}{\frac{1}{v} - K } + m_{t-1} + 1\\
    & \geq m_{t-1}(1+ \frac{v}{1-Kv}) >  m(1+ \frac{v}{1-Kv})
\end{align*}

By choosing $m = \lfloor \frac{16\ln T }{ \Delta_i ^2}* \frac{1-Kv}{1-(K-1)v} \rfloor$,
\begin{align*}
    \Delta_i - 2\sqrt{\frac{\ln T}{n_{t-1}(i)}} 
    =4\sqrt{\frac{\ln T }{m(1+ \frac{v}{1-Kv})}} - 2\sqrt{\frac{\ln T}{n_{t-1}(i)}}
    > 2\sqrt{\frac{\ln T}{n_{t-1}(i)}} 
\end{align*}

Again replace the above result in the probability bound and use  Lemma~\ref{lem:random hoeff 1} as before, we get 
\begin{align*}
    &Prob\left[ \hat{\mu}_t(i) -  \mu(i) \geq \Delta_i - 2\sqrt{\frac{\ln T}{n_{t-1}(i)}} \quad and \quad m_{t-1}>m \right]\\
    &\leq Prob\left[ \hat{\mu}_t(i) -  \mu(i) \geq 2\sqrt{\frac{\ln T}{n_{t-1}(i)}}  \right]
    \leq \frac{(1-Kv)}{T} + \order(\frac{K}{T^2})
\end{align*}

\textbf{Therefore, we conclude bound for  $i\neq i^{*}$} that 
\begin{align*}
\E [m_T(i)] 
&\leq  \frac{16\ln T}{\Delta_i ^2}\left( \frac{1-Kv}{1-(K-1)v} \right) +\sum_{t \in \calI, t> \calI[m]} \left(2\frac{(1-Kv)}{T} + 2\order(\frac{K}{T^2}) \right)\\
&\leq  \frac{16\ln T}{\Delta_i ^2}\left( \frac{1-Kv}{1-(K-1)v} \right) + 2(1-Kv)^2 + 2\frac{(1-Kv)}{Tv} + 2\order(\frac{K}{T}) \\
&\leq  \frac{16\ln T}{\Delta_i ^2}\left( \frac{1-Kv}{1-(K-1)v} \right) +  2(1-Kv)^2 + \order(1) 
\end{align*}
\textbf{Now we can get the total regret is: }
\begin{align*}
    Reg_T 
    & =\sum_{i \neq i^*} \Delta_i \E \left[m_T(i) \right]\\
    &\leq \sum_{i:\Delta_i > 0}  \left[ \frac{16\ln T}{\Delta_i}\left( \frac{1-Kv}{1-(K-1)v} \right)+2(1-Kv)^2\Delta_i  \right] + \order (K) 
\end{align*}

This bound is not always tight, because when $\Delta \rightarrow \order{(\frac{1}{T})}$ and $v \ll \frac{1}{K}$, this bound will become linear. Therefore, for any $\Delta \in [0,1]$ we can further write that as
\begin{align*} 
    Reg_T 
    &= \sum_{\Delta_i \leq \Delta} \Delta_i \E \left[m_T(i)\right] + \sum_{\Delta_i > \Delta} \Delta_i \E \left[m_T(i)\right] \\
    &\leq \Delta (T - K\lfloor Tv \rfloor) + \sum_{\Delta_i > \Delta} \left[ \frac{16\ln T}{\Delta_i}\left( \frac{1-Kv}{1-(K-1)v} \right) +2(1-Kv)^2\Delta_i \right] + \order (K) 
\end{align*}

By choosing $\Delta = \sqrt{\frac{K\ln T}{T}}$, we got the worst case guarantee,
\begin{align*}
    Reg_T 
    & \leq T\Delta + \sum_{\Delta_i > \Delta} \left[ \frac{16\ln T}{\Delta_i}+2\Delta_i   \right]  + \order (K) \\
    & \leq \order (\sqrt{TK\ln T} + K\ln T ) \\
\end{align*}

\subsection{Proof for algorithm 2 }

\subsubsection{Notations}
Denote the distribution over arms at time $t$ as $p_t$ where $p_t(\argmax_{i \in [K]}UCB_t(i))=(1-Kv)+v$ and $p_t(i) = v, \forall i \in [K]\setminus{\argmax_{i \in [K]}UCB_t(i)}$. And the best distribution as $p^*$ where $p^*(i^*)=(1-Kv)+v$ and $p^*(i) = v, \forall i \in [K]\setminus{i^*}$.

\subsubsection{Main Proof}
\textbf{First we rewrite this regret in the form of variable $\Delta_i$ and $m_T$},
\begin{align*}
    & \E \left[ \sum_{t=1}^T \E_{i_t\sim p^*}[r_t(i_t)] - \E_{i_t \sim p_t}[r_t(i_t)] \right] \\
    & = \E \left[ \sum_{t=1}^T \E_{i_t\sim p^*}[\mu(i_t)] - \E_{i_t \sim p_t}[\mu(i_t)] \right] \\
    &= \E  \sum_{t=1}^T \left[ (1-(K-1)v)\mu(i^*) + v\sum_{i\neq i^*}\mu(i) - p_t(i^*)\mu(i^*) - \sum_{i\neq i^*}p_t(i)\mu(i) \right]\\
    &= \E \sum_{t=1}^T \left[ (1-p_t(i^*))\mu(i^*)- \sum_{i\neq i^*}p_t(i)\mu(i) + v\sum_{i\neq i^*}(\mu(i) - \mu(i^*))  \right]\\ 
    &= \E  \sum_{t=1}^T \left[\sum_{i\neq i^*} p_t(i)\Delta_i - v\sum_{i\neq i^*} \Delta_i \right]\\
    &= \sum_{i\neq i^*} \Delta_i \E \left[\sum_{t=1}^T p_t(i)\right] -  vT\sum_{i\neq i^*} \Delta_i \\
    &= \sum_{i\neq i^*} \Delta_i \E \left[\sum_{t=1}^T \one\{i_t =i \}\right] -  vT\sum_{i\neq i^*} \Delta_i \\
    &= \sum_{i\neq i^*} \Delta_i \E \left[n_T(i)\right] -  vT\sum_{i\neq i^*} \Delta_i 
\end{align*}
Notice here the expectation is regarding to the both the randomness of environment and learner's choice of $i_t$, which is a bit different from previous proof. The penultimate equality is due to $\E_{\text{learner at t}}[\one\{i_t =i \}] = p_t(i)$ and the linearity of expectation.

\textbf{Next we want to bound $\E\left[n_T\right] $ following the similar idea as in the original UCB paper.} 
\begin{align*}
\E[n_T(i)] &= n + \sum_{t=n+1}^T Prob\left[ (i_t = i)\quad  and \quad n_{t-1}>n \right] \text{  ($n$ here is simply for analysis)} \\
& \leq n + \sum_{t=n+1}^T \left[ (1-Kv)*\underbrace{Prob(UCB(i)>UCB(i^{*})\quad and \quad n_{t-1}(i)>n)}_{\text{Term1}} + v  \right]\\ 
\end{align*}

Here $n$ can be any non-negative integer. In the later analysis, choice of $n$ helps us to get a tighter bound.

\textbf{Now we analyze the \textsc{Term1}} using almost the same technique as the proof for algorithm 1.

\begin{align*}
    \textsc{Term1}
    &\leq Prob\left[ UCB_{t}(i^{*}) < \mu(i^{*}) \right] + Prob\left[ UCB_{t}(i) > \mu(i^{*}) \quad and \quad  n_{t-1}>n \right] \\
    & \leq Prob\left[\mu(i^{*}) - \hat{\mu}_t(i^{*}) \geq 2\sqrt{\frac{\ln T}{n_{t-1}(i^{*})}} \right] \\
    & \quad + Prob\left[ \hat{\mu}_t(i) -  \mu(i) \geq \Delta_i - 2\sqrt{\frac{\ln T}{n_{t-1}(i)}} \quad and \quad  n_{t-1}>n \right] \\   
\end{align*}

First, observe that $Prob\left[\mu(i^{*}) - \hat{\mu}_t(i^{*}) \geq 2\sqrt{\frac{\ln T}{n_{t-1}(i^{*})}} \right]$ has nothing to do with $n$, thus we can directly apply Lemma~\ref{lem:random hoeff 1} to get the upper  bound. Again we want to know the $[k_s,k_e]$. Since there is no prescheduling here, the interval is $[1,t-K]$. So the upper bound is 
\begin{align*}
     \sum_{k=1}^{t-K} \frac{1}{T^2} < \frac{1}{T} 
\end{align*}

Then we address $Prob\left[ \hat{\mu}_t(i) -  \mu(i) \geq \Delta_i - 2\sqrt{\frac{\ln T}{n_{t-1}(i)}} \quad and \quad  n_{t-1}>n \right]$. Again we want to use the Lemma~\ref{lem:random hoeff 1}, but we need to choose $n$ at first. The reason we want to choose $n$ is that we consider the bound to be very loose in the first $n$ epochs, so we can directly bound the probability by $1$. 

We also need to consider the extreme case where $v \rightarrow \frac{1}{K}$, so $(1-Kv) \rightarrow 0$. In that case, the choice of the arm is random and irrelevant to the UCB bound, so we can simply choose $n=0$ and all the probabilities will be naturally bounded by $1$. 

Therefore, we will consider two cases, $n = [\frac{16\ln T}{\Delta_i ^2}]$ and $n = 0$.

When $n = 0$, we simply bound the probability by $1$. 

If $n = [\frac{16\ln T }{\Delta_i ^2}]$, we observe that 
\begin{align*}
    \Delta_i - 2\sqrt{\frac{\ln T}{n_{t-1}(i)}} 
    =4\sqrt{\frac{\ln T}{n}} - 2\sqrt{\frac{\ln T}{n_{t-1}(i)}}
    > 2\sqrt{\frac{\ln T}{n_{t-1}(i)}} 
\end{align*}
We can again apply Lemma~\ref{lem:random hoeff 1} as before and get 
\begin{align*}
    &Prob\left[ \hat{\mu}_t(i) -  \mu(i) \geq \Delta_i - 2\sqrt{\frac{\ln T}{n_{t-1}(i)}} \quad and \quad n_{t-1}>n \right]\\
    &\leq Prob\left[ \hat{\mu}_t(i) -  \mu(i) \geq 2\sqrt{\frac{\ln T}{n_{t-1}(i)}}  \right]
    \leq \frac{1}{T}
\end{align*}

\textbf{Therefore, combine the two cases, we conclude bound for  $i\neq i^{*}$} 
\begin{align*}
    \E[n_T(i)] \leq \min\left\{ \frac{16\ln T}{\Delta_i ^2} + (1-Kv), (1-Kv)T \right\} +  vT
\end{align*}

\textbf{Now we can get the total regret is: }
\begin{align*}
    Reg_T  &= \sum_{i\neq i^*} \Delta_i \E \left[n_T(i)\right] -  vT\sum_{i\neq i^*} \Delta_i \\
    & \leq \sum_{i:\Delta_i > 0}  \left[ \min \left\{\frac{16\ln T}{\Delta_i}+(1-Kv)\Delta_i, (1-Kv)\Delta_iT \right\}\right] 
\end{align*}

This bound is not always tight, because when $\Delta \rightarrow \order{(\frac{1}{T})}$ and $v \ll \frac{1}{K}$, this bound will become linear. Therefore, for any $\Delta \in [0,1]$ we can further write that as
\begin{align*}
    Reg_T 
    &\leq \sum_{\Delta_i \leq \Delta} \Delta_i \E \left[n_T(i)\right] + \sum_{\Delta_i > \Delta} \Delta_i \E \left[n_T(i)\right] -  vT\sum_{\Delta_i > \Delta} \Delta_i \\
    &\leq \Delta T +\sum_{i:\Delta_i > \Delta}  \left[ \min \left\{  \left(\frac{16\ln T}{\Delta_i}+(1-Kv)\Delta_i \right), (1-Kv)\Delta_iT \right\} \right] 
\end{align*}
By choosing $\Delta = \sqrt{\frac{K\ln T}{T}}$, we got the worst case guarantee,
\begin{align*}
    Reg_T 
    & \leq T\Delta + \sum_{\Delta_i > \Delta} \left[ \frac{16\ln T}{\Delta_i}+2\Delta_i   \right]  + \order (K) \\
    & \leq \order (\sqrt{TK\ln T} + K\ln T ) \\
\end{align*}

\end{document}